\documentclass[letterpaper]{article}
\usepackage{fullpage,amsmath,amsthm,amssymb}
\usepackage{xcolor}
\usepackage[normalem]{ulem}

\usepackage[noend,algoruled,linesnumbered,]{algorithm2e}

\date{}

\newcommand{\cH}{\mathcal{H}}
\newcommand{\cX}{\mathcal{X}}
\newcommand{\cD}{\mathcal{D}}
\newcommand{\cA}{\mathcal{A}}

\newcommand{\bB}{\mathbf{B}}
\newcommand{\bC}{\mathbf{C}}
\newcommand{\bT}{\mathbf{T}}
\newcommand{\bx}{\mathbf{x}}
\newcommand{\by}{\mathbf{y}}
\newcommand{\bh}{\mathbf{h}}
\newcommand{\E}{\mathbb{E}}

\newcommand{\bD}{\mathbf{D}}
\newcommand{\bH}{\mathbf{H}}
\newcommand{\bX}{\mathbf{X}}
\newcommand{\bY}{\mathbf{Y}}
\newcommand{\eps}{\varepsilon}
\newcommand{\randH}{\mathbf{h}}
\newcommand{\bn}{\mathbf{n}}
\newcommand{\bS}{\mathbf{S}}

\DeclareMathOperator{\er}{er}
\DeclareMathOperator{\argmin}{argmin}

\newtheorem{observation}{Observation}
\newtheorem{lemma}{Lemma}
\newtheorem{theorem}{Theorem}

\newtheorem{innercustomlem}{Restatement of Lemma}
\newenvironment{customlem}[1]
  {\renewcommand\theinnercustomlem{#1}\innercustomlem}
  {\endinnercustomlem}

\newtheorem{innercustomobs}{Restatement of Observation}
\newenvironment{customobs}[1]
  {\renewcommand\theinnercustomobs{#1}\innercustomobs}
  {\endinnercustomobs}

\title{Revisiting Agnostic PAC Learning}

\author{Steve Hanneke\thanks{\texttt{steve.hanneke@gmail.com}}\\Purdue University \and Kasper Green Larsen\thanks{\texttt{larsen@cs.au.dk}}\\Aarhus University \and Nikita Zhivotovskiy\thanks{\texttt{zhivotovskiy@berkeley.edu}}\\UC Berkeley}

\begin{document}

\date{}
\maketitle

\begin{abstract}
PAC learning, dating back to Valiant'84 and Vapnik and Chervonenkis'64,'74, is a classic model for studying supervised learning. In the \emph{agnostic} setting, we have access to a hypothesis set $\cH$ and a training set of labeled samples $(x_1,y_1),\dots,(x_n,y_n) \in \cX \times \{-1,1\}$ drawn i.i.d.\ from an unknown distribution $\cD$. The goal is to produce a classifier $h : \cX \to \{-1,1\}$ that is competitive with the hypothesis $h^\star_\cD \in \cH$ having the least probability of mispredicting the label $y$ of a new sample $(x,y)\sim \cD$.

Empirical Risk Minimization (ERM) is a natural learning algorithm, where one simply outputs the hypothesis from $\cH$ making the fewest mistakes on the training data. This simple algorithm is known to have an optimal error in terms of the VC-dimension of $\cH$ and the number of samples $n$.

In this work, we revisit agnostic PAC learning and first show that ERM is in fact sub-optimal if we treat the performance of the best hypothesis, denoted $\tau:=\Pr_\cD[h^\star_\cD(x) \neq y]$, as a parameter. Concretely we show that ERM, and any other proper learning algorithm, is sub-optimal by a $\sqrt{\ln(1/\tau)}$ factor. We then complement this lower bound with the first learning algorithm achieving an optimal error for nearly the full range of $\tau$. Our algorithm introduces several new ideas that we hope may find further applications in learning theory.
\end{abstract}

\section{Introduction}
One of the most basic theoretical models for studying binary classification in a supervised learning setup, is the Probably Approximately Correct (PAC) learning framework of Valiant~\cite{valiant1984theory}, and Vapnik and Chervonenkis~\cite{vapnik1964class, vapnik74theory}. In this framework, a training data set consists of $n$ i.i.d.\ samples $\bS = \{(\bx_i,\by_i)\}_{i=1}^n$ from an unknown data distribution $\cD$ over $\cX \times \{-1,1\}$. Here $\cX$ is an input domain and $\{-1,1\}$ are the two possible labels. The goal is to design a learning algorithm $\cA$, that on a training set $\bS$, produces a classifier/hypothesis $h_\bS : \cX \to \{-1,1\}$ minimizing the probability of mispredicting the label of a fresh sample from $\cD$, denoted by $\er_\cD(h) := \Pr_\cD[h(\bx) \neq \by]$.

In the PAC learning framework, the algorithm $\cA$ is further given a hypothesis set $\cH \subseteq \cX \to \{-1,1\}$, and the performance of the hypothesis $h_\bS$ produced by $\cA$ must be competitive with the best hypothesis $h^\star_\cD$ in $\cH$, where $h^\star_\cD := \argmin_{h \in \cH} \er_\cD(h)$ (breaking ties arbitrarily). Classic work on PAC learning distinguishes two important cases, namely \emph{realizable} and \emph{agnostic} learning. In the realizable setting, it is assumed that $\er_\cD(h^\star_\cD) = 0$, i.e.\ that there is a hypothesis in $\cH$ perfectly classifying all data. Here the goal is to achieve $\er_\cD(h_\bS) \leq \eps$ for $\eps$ going to $0$ as fast as possible with $n$. In the often more realistic setup of agnostic learning, the goal is instead to guarantee $\er_\cD(h_\bS) \leq \er_\cD(h^\star_\cD) + \eps$, thus being competitive with $h^\star_\cD$.

\paragraph{Realizable setting.}
The realizable setting is by now very well understood, in particular following a surge of results over the past few years. The most natural learning algorithm in this setting is Empirical Risk Minimization (ERM), that simply outputs an arbitrary hypothesis $h_\bS \in \cH$ that achieves $\er_\bS(h_\bS)=0$. Here, and throughout the paper, we let $\er_S(h)$ for a set of samples $S$ denote $\Pr_S[h(\bx) \neq \by]$, and when we subscript a probability by $S$, we let $(\bx,\by)$ be a uniform random sample from $S$. Note that a hypothesis $h_\bS$ with $\er_\bS(h_\bS)=0$ is guaranteed to exist since $h^\star_\cD$ is one such hypothesis. Classic work~\cite{vapnik74theory,vapnik:estimation,blumer1989learnability} shows that ERM guarantees, with probability $1-\delta$ over $\bS \sim \cD^n$, that $\er_\cD(h_\bS) = O((d \ln(n/d) + \ln(1/\delta))/n)$. Here $d$ denotes the VC-dimension~\cite{vapnik71uniform} of $\cH$ and is defined as the largest number of points $x_1,\dots,x_d \in \cX$ for which $\cH$ can generate all $2^d$ possible labelings of $x_1,\dots,x_d$. This \emph{sample complexity} is known~\cite{haussler1994predicting,auer2007new,simon2015almost,bousquet2020proper} to be optimal for any \emph{proper} learning algorithm, i.e.\ there exists an input domain $\cX$, hypothesis set $\cH$ of VC-dimension $d$ and a data distribution $\cD$, such that any $\cA$ that outputs a hypothesis $h_\bS$ from $\cH$ must have $\er_{\cD}(h_\bS) = \Omega((d \ln(n/d) + \ln(1/\delta))/n)$ with probability at least $\delta$. Determining the optimal sample complexity for \emph{improper} learning algorithms, i.e.,\ algorithms that are allowed to output an arbitrary hypothesis $h_\bS : \cX \to \{-1,1\}$, and not only hypotheses from $\cH$, was a major open problem for decades. Finally, in work by Hanneke~\cite{hanneke2016optimal}, building on ideas of Simon~\cite{simon2015almost}, an optimal learning algorithm guaranteeing $\er_\cD(h_\bS) = O((d+\ln(1/\delta))/n)$ with probability $1-\delta$ was finally developed. This matches previous lower bounds~\cite{blumer1989learnability,ehrenfeucht1989general} and thus settled the complexity of realizable PAC learning. Over the past few years, there have been several works proving optimality of other and arguably simpler learning algorithms, including for the practical heuristic bagging~\cite{larsen2023baggingCOLT, Breiman2004BaggingP}, a variant of the one-inclusion graph learning algorithm~\cite{Aden-AliFOCS23, haussler1994predicting}, and most recently for a simple majority vote among three ERM classifiers~\cite{maj3}.

\paragraph{Agnostic setting.}
ERM is also a very natural learning algorithm in the agnostic setting. Instead of outputting a hypothesis $h_\bS$ with $\er_\bS(h_\bS)=0$ (which might not exist), ERM instead outputs the hypothesis $h^\star_\bS = \argmin_{h \in \cH} \er_\bS(h)$ achieving the best performance on the training data (breaking ties arbitrarily). This strategy is well understood and is known~\cite{HAUSSLER199278} to guarantee $\er_\cD(h_\bS) = \er_\cD(h^\star_\cD) + O(\sqrt{(d + \ln(1/\delta))/n})$. Note that always use $h^\star_\cD$ to denote the hypothesis $\argmin_h \er_\cD(h)$. Unlike the realizable setting, there is a matching lower bound~\cite{bookvc} (Theorem 5.2) even for improper learning algorithms. Thus in contrast to the realizable setting, simple ERM is provably optimal. While this might seem the end of the story, the picture is however more complicated. In particular, one would expect there to be some form of transition between the agnostic and realizable setting, i.e.\ for sufficiently small $\tau = \er_\cD(h^\star_\cD)$, ERM must become sub-optimal. The bounds with the explicit dependence on $\tau$ are quite standard in the literature and are sometimes called the \emph{first-order bounds}, especially in the contexts of online learning and optimization. For the state-of-the-art upper and lower bounds in the agnostic PAC learning setup, we refer to \cite[Corollary 5.3]{boucheron2005theory} and the corresponding lower bounds in \cite[Chapter 14]{DevroyeGyorfiLugosi1996} and in \cite{audibert2009fast}. Since we revisit ERM and state its sample complexity also as a function of $\tau$, we start with the following upper bounds (with $0 \ln(1/0)=0$):
\begin{theorem}[ERM Theorem, derived from~\cite{lls}]
\label{thm:ERM}
For any input domain $\cX$, hypothesis set $\cH$ of VC-dimension $d$, number of samples $n$, distribution $\cD$ over $\cX \times \{-1,1\}$ and any $0 < \delta < 1$, it holds with probability at least $1-\delta$ over a sample $\bS \sim \cD^n$ that every hypothesis $h \in \cH$ satisfies
\[
|\er_{\cD}(h)-\er_{\bS}(h)| = O\left(\sqrt{\frac{\er_{\cD}(h)(d \ln(1/\er_{\cD}(h)) + \ln(1/\delta))}{n}} + \frac{d \ln(n/d)+\ln(1/\delta)}{n} \right).
\]
In particular, this implies that running ERM returns a hypothesis $h_\bS \in \cH$ satisfying
\[
\er_{\cD}(h_\bS) = \tau + O\left(\sqrt{\frac{\tau(d \ln(1/\tau) + \ln(1/\delta))}{n}} + \frac{d \ln(n/d)+\ln(1/\delta)}{n} \right),
\]
where $\tau = \er_{\cD}(h^\star_\cD)$.
\end{theorem}
Observe the $\sqrt{\tau \ln(1/\tau)}$ dependency in $\er_\cD(h_\bS)$ that smoothly interpolates between the agnostic and realizable setting. By the lower bounds in \cite[Chapter 14]{DevroyeGyorfiLugosi1996}, we have that any learning algorithm $\cA$ produces with probability at least $\delta$ a hypothesis $h_\bS$ with 
\[
\er_{\cD}(h_\bS) = \tau + \Omega\left(\sqrt{\frac{\tau(d + \ln(1/\delta))}{n}} + \frac{d + \ln(1/\delta)}{n}\right).
\]
Thus there is a $\sqrt{\ln(1/\tau)}$ gap between ERM and the lower bound. Furthermore, and unlike the realizable setting, there are no known algorithms that bridge this gap and we have no proof that optimal algorithms need to be improper (except when $\tau = 0$).

\paragraph{Our Contributions.}
In this work, we close this gap for almost the full range of $\tau$. First, we prove that any proper learning algorithm must incur this $\sqrt{\ln(1/\tau)}$ factor in its sample complexity:
\begin{theorem}
\label{thm:mainlower}
There is a constant $C>0$ such that for any VC-dimension $d$, number of samples $n$ and $\tau$ satisfying $C d \ln(n/d)/n \leq \tau \leq 1/C$, there is an input domain $\cX$ and hypothesis set $\cH$ of VC-dimension $d$, satisfying that for every \textbf{proper} learning algorithm $\cA$, there is a distribution $\cD$ over $\cX \times \{-1,1\}$ such that:
\begin{enumerate}
    \item There is a hypothesis $h \in \cH$ with $\er_\cD(h)=\tau$.
    \item With probability at least $1/16$ over a sample $\bS \sim \cD^n$, it holds that the hypothesis $h_\bS \in \cH$ produced by $\cA$ on $\bS$ has $\er_\cD(h_\bS) = \tau + \Omega(\sqrt{\tau d \ln(1/\tau)/n})$.
\end{enumerate}
\end{theorem}
Motivated by this lower bound, we design a new improper learning algorithm that avoids this $\sqrt{\ln(1/\tau)}$ penalty and achieves an optimal sample complexity except for very small values of $\tau$:
\begin{theorem}
    \label{thm:mainupper}
For any input domain $\cX$, hypothesis set $\cH$ of VC-dimension $d$, number of samples $n$, distribution $\cD$ over $\cX \times \{-1,1\}$ and any $0 < \delta < 1$, there is an algorithm, \textsc{DisagreeingExperts}, that when given samples $\bS \sim \cD^n$ and failure probability $\delta$, returns with probability at least $1-\delta$ a hypothesis $h_\bS : \cX \to \{-1,1\}$ satisfying
\[
\er_{\cD}(h_\bS) = \tau + O\left( \sqrt{\frac{\tau(d + \ln(1/\delta))}{n} } + \frac{\ln^{5}(n/d)(d + \ln(1/\delta))}{n}\right),
\]
where $\tau = \er_{\cD}(h^\star_\cD)$.
\end{theorem}
This is the first known learning algorithm to provably outperform ERM in the agnostic setting. Furthermore, we stress that despite the recent progress on realizable PAC learning, none of the ideas in those works seem to generalize easily to the agnostic setting. Instead, our algorithm is based on a new paradigm of recursively training pairs of nearly optimal classifiers that disagree in many of their predictions. We elaborate on this new approach in Section~\ref{sec:overview} and hope it may find further applications in learning theory.

\subsection{Proof Overview}
\label{sec:overview}
In this section, we present the high level ideas of both our new agnostic PAC learning algorithm, \textsc{DisagreeingExperts}, as well as our lower bound for proper learners. We begin with the upper bound.

\paragraph{New algorithm.}
Our improved algorithm relies on several new insights regarding Empirical Risk Minimization. To set the stage for describing these ideas, consider a data distribution $\cD$ over $\cX \times \{-1,1\}$ and let $\tau = \er_\cD(h^\star_\cD)$. If we run ERM on a data set $\bS \sim \cD^n$ of size $n$, then by the ERM Theorem (Theorem~\ref{thm:ERM}), this ensures that for sufficiently large constant $c>0$, ERM will not return a hypothesis $h$ with $\er_\cD(h) > \tau + c\sqrt{\tau d \ln(1/\tau) /n}$ (let us ignore $\delta$ and the additive $d \ln(n/d)/n$ term in the ERM Theorem for simplicity). This is, of course, a $\sqrt{\ln(1/\tau)}$ factor worse than what we are aiming for. To improve this bound, we show that there is always a \emph{win-win} situation we may exploit to shave the $\sqrt{\ln(1/\tau)}$ factor.

To understand this win-win scenario, consider the set $\bar{\cH} \subseteq \cH$ of near-optimal hypotheses $h$ with $\er_\cD(h) \leq \tau + c\sqrt{\tau d \ln(1/\tau) /n}$, i.e.,\ the hypotheses that might be returned by a typical execution of ERM. These are what we think of as \emph{experts} in our algorithm name \textsc{DisagreeingExperts}. In the proof of the ERM Theorem, the basic idea is to union bound over all  $h \in \bar{\cH}$ (with a chaining argument), to show that $|\er_\cD(h)-\er_\bS(h)|=O(\sqrt{\tau d \ln(1/\tau)/n})$ for all $h \in \bar{\cH}$ simultaneously and thus returning the hypothesis $h^\star_\bS$ with smallest error on $\bS$ is a good strategy. Our first new insight is, that if the hypotheses in $\bar{\cH}$ are sufficiently similar, then this union bound improves for $\bar{\cH}$. Concretely, assume that all pairs of hypotheses $h_1,h_2 \in \bar{\cH}$ have $\Pr_\cD[h_1(\bx)\neq h_2(\bx)] = O(\tau/\ln(1/\tau))$. We argue that this implies that all hypotheses in $\bar{\cH}$ satisfy the stronger guarantee that $|\er_\bS(h)-\er_\cD(h)| = O(\sqrt{\tau d/n})$ rather than just $O(\sqrt{\tau d \ln(1/\tau)/n})$, hence improving the accuracy obtained from ERM. The intuitive reason for this improvement is, that when all the hypotheses in $\bar{\cH}$ are very similar, it suffices to bound $|\er_\bS(h)-\er_\cD(h)|$ for one $h \in \bar{\cH}$ and to bound $|(\er_\bS(h')-\er_\bS(h)) - (\er_\cD(h')-\er_\cD(h))|$ for all other $h' \in \bar{\cH}$. Since $|\er_\cD(h')-\er_\cD(h)|=O(\Pr_\cD[h(\bx)\neq h'(\bx)]) = O(\tau/\ln(1/\tau))$, we get stronger concentration on $|(\er_\bS(h')-\er_\bS(h)) - (\er_\cD(h')-\er_\cD(h))|$ than each individual $|\er_\bS(h')-\er_\cD(h')|$.

Unfortunately, we have no guarantee that all pairs of hypotheses $h_1,h_2 \in \bar{\cH}$ have $\Pr_\cD[h_1(\bx)\neq h_2(\bx)] = O(\tau/\ln(1/\tau))$. Our next contribution is thus to find a way of exploiting the existence of two near-optimal hypotheses $h_1,h_2$ with $\Pr_\cD[h_1(\bx)\neq h_2(\bx)] = \Omega(\tau/\ln(1/\tau))$ (i.e.,\ a pair of disagreeing experts). Here we show that the conditional distribution of a sample $(\bx,\by)$ from $\cD$ with $h_1(\bx)=h_2(\bx)$ is "easier"\footnote{A similar argument was used in \cite{bousquet2021fast, puchkin2021exponential} in the context of classification with an abstention option. The authors also use the disagreement sets of what we call the \emph{experts} --- candidates for being an output of a typical ERM. However, the authors of \cite{bousquet2021fast, puchkin2021exponential} focus on either abstaining or learning the labels of the set of disagreements of pairs of experts, while in this work we use the fact that the conditional distribution of the set where two experts agree is "easier".
} than the distribution $\cD$. In more detail, we know that $\er_\cD(h_j) = \tau + O(\sqrt{\tau d\ln(1/\tau)/n})$ for $j=1,2$. Since precisely one of them errs whenever $h_1(x) \neq h_2(x)$, and both or none err when $h_1(x)=h_2(x)$, we have that $\er_\cD(h_1) + \er_\cD(h_2) = \Pr_\cD[h_1(\bx) \neq h_2(\bx)] + 2\Pr_\cD[h_1(\bx) \neq \by \wedge h_1(\bx) = h_2(\bx)]$. Since they are both near-optimal, this implies $\Pr_\cD[h_1(\bx) \neq \by \wedge h_1(\bx)=h_2(\bx)] = \tau + O(\sqrt{\tau d \ln(1/\tau)/n}) - \Omega(\tau/\ln(1/\tau))$. This is $\tau - \Omega(\tau/\ln(1/\tau))$ for $\tau$ sufficiently large (this assumption is one of the causes of the additive $\ln^5(n/d)(d + \ln(1/\delta))/n$ term in our upper bounds). Rewriting this also gives $\Pr_\cD[h_1(\bx) \neq \by \mid h_1(\bx)=h_2(\bx)] = \Pr_\cD[h_1(\bx) = h_2(\bx)]^{-1} (\tau - \Omega(\tau/\ln(1/\tau)))$. Recalling that precisely one of $h_1$ and $h_2$ errs when they disagree, and that they are both near-optimal implies $\Pr_\cD[h_1(\bx) = h_2(\bx)] \geq 1- O(\tau)$ and thus $\Pr_\cD[h_1(\bx) \neq \by \mid h_1(\bx)=h_2(\bx)] = (1+O(\tau))(\tau - \Omega(\tau/\ln(1/\tau))) = \tau - \Omega(\tau/\ln(1/\tau))$. What we have just argued is, that under the conditional distribution $\cD_=$ of a sample $(\bx,\by) \sim \cD$ with $h_1(\bx)=h_2(\bx)$, there is a hypothesis $h^\star_{\cD_=} \in \cH$ with $\er_{\cD_=}(h^\star_{\cD_=}) = \tau - \Omega(\tau/\ln(1/\tau))$ (in particular, both $h_1$ and $h_2$ have this property). The distribution $\cD_=$ is thus somewhat easier than $\cD$ since the optimal error under $\cD$ is $\tau$.

Our next idea is to repeat the above argument recursively in order to drive $\er_{\cD_=}(h^\star_{\cD_=})$ further down. More formally, if we can again find a pair of disagreeing experts $h_1,h_2$ for the distribution $\cD_=$ and repeat this $t$ times, then we end up with a list of pairs $(h_1^1,h_2^1),\dots,(h_1^t,h_2^t)$ such that under the distribution $\cD_=$ of a sample $(\bx,\by) \sim \cD$ conditioned on $\forall i : h^i_1(\bx)=h^i_2(\bx)$, we have $\er_{\cD_=}(h^\star_{\cD_=}) \leq \tau(1-1/\ln(1/\tau))^t$. After $t = O(\ln(1/\tau)\ln \ln(1/\tau))$ iterations, we have ensured $\er_{\cD_=}(h^\star_{\cD_=}) \leq \tau/\ln(1/\tau)$. Empirical Risk Minimization on samples $\bS_=$ with $\forall i : h^i_1(\bx)=h^i_2(\bx)$ then gives a hypothesis with $\er_{\cD_=}(h^\star_{\bS_=}) = \er_{\cD_=}(h^\star_{\cD_=}) + O(\sqrt{(\tau/\ln(1/\tau)) d \ln(\ln(1/\tau)/\tau)/n}) = \er_{\cD_=}(h^\star_{\cD_=}) + O(\sqrt{\tau d/n})$.

What remains is to handle samples with $h^i_1(\bx) \neq h^i_2(\bx)$ for some $i$. We let $\cD_{\neq}$ denote the distribution $\cD$ conditioned on such a sample. Our key observation is that we can control the probability of receiving such a sample. Concretely, we show that $\Pr_\cD[\exists i : h^i_1(\bx) \neq h^i_2(\bx)] = \Theta(\tau)$. We thus expect to see $\Theta(\tau n)$ samples, denoted $\bS_{\neq}$, from $\cD_{\neq}$ in a training set $\bS \sim \cD^n$. A completely naive invocation of the ERM Theorem, only assuming $\tau = O(1)$, shows that we find a hypothesis $h^\star_{\bS_{\neq}}$ with $\er_{\cD_{\neq}}(h^\star_{\bS_{\neq}}) = \er_{\cD_{\neq}}(h^\star_{\cD_{\neq}}) + O(\sqrt{d/|\bS_{\neq}|}) = \er_{\cD_{\neq}}(h^\star_{\cD_{\neq}}) + O(\sqrt{d/(\tau n)})$. Note that the dependency on $\tau$ is very bad for this hypothesis, i.e., a $\sqrt{1/\tau}$ rather than $\sqrt{\tau}$. However, since samples with $h^i_1(\bx) \neq h^i_2(\bx)$ are so rare, this turns out to be sufficient.

We now have all the ingredients for our algorithm. If we have obtained the pairs of disagreeing experts $(h^1_1,h^1_2),\dots,(h^t_1,h^t_2)$ and the two hypotheses $h^\star_{\bS_=}$ and $h^\star_{\bS_{\neq}}$, our final classifier does as follows on a new point $x \in \cX$ without a label: First, it checks whether there is a pair with $h^i_1(x) \neq h^i_2(x)$. If so, it returns $h^\star_{\bS_{\neq}}(x)$. Otherwise, it returns $h^\star_{\bS_=}(x)$. If $p$ denotes $\Pr_\cD[\exists i : h_1^i(\bx) \neq h_2^i(\bx)]$, then $p = O(\tau)$ and our final classifier $h_\bS$ satisfies
\begin{eqnarray*}
    \er_\cD(h_\bS) &=& p \er_{\cD_{\neq}}(h^\star_{\bS_{\neq}}) + (1-p)\er_{\cD_=}(h^\star_{\bS_=}) \\
    &=& p\er_{\cD_{\neq}}(h^\star_{\cD_{\neq}})+ O(p\sqrt{d/(\tau n)}) + (1-p)\er_{\cD_=}(h^\star_{\cD_=}) + O(\sqrt{\tau d/n}) \\
    &=& p \er_{\cD_{\neq}}(h^\star_{\cD}) + (1-p)\er_{\cD_=}(h^\star_{\cD}) + O(\sqrt{\tau d/n}) \\
    &=& \er_\cD(h^\star_\cD) + O(\sqrt{\tau d/n}).
\end{eqnarray*}
This completes the high level description of the key ideas in our new algorithm. Let us finally remark that we clearly do not have $t$ training sets of size $n$ each for training $h^i_1,h^i_2$ for $i=1,\dots,t$. Instead, we allocate around $n/t$ samples for each iteration. This of course reduces the performance of any estimates based on ERM. However, we can show that this only matters for very small values of $\tau$ and thus is a second source of the additive $\ln^5(n/d)(d + \ln(1/\delta))/n$ term.

\paragraph{Lower bound for proper learners.}
Our lower bound proof is quite simple. Assume we wish to prove a lower bound on the error of a proper learner when the hypothesis set has VC-dimension $d$, we have $n$ samples and $\er_{\cD}(h^\star_\cD)=\tau$ for some $\tau$. We construct an instance where the input domain $\cX$ is the discrete set $x_1,\dots,x_u$ with $u \approx d/\tau$. We let the hypothesis set $\cH$ consist of all hypotheses returning $-1$ on precisely $d$ of the $u$ points. Finally, the unknown concept we are trying to learn is the all-1 concept. This hypothesis class is routinely used in the existing lower bounds, and in fact corresponds to the hardest case under Massart's noise condition~\cite{massart2006risk, raginsky2011lower, hanneke2016refined, zhivotovskiy2018localization}. Note that $\mathcal{H}$ does not contain the all-1 concept; thus, we cannot simply choose the proper learner that always outputs this concept.

Assume we have some proper learning algorithm $\cA$ for this hypothesis set and input domain $\cX$. We now consider a number of different data distributions $\cD_1,\cD_2,\dots,$ and argue that there is at least one of the distributions $\cD_i$ under which $\cA$ often (with constant probability) produces a hypothesis $h_\bS$ with $\er_{\cD_i}(h_\bS) = \tau + \Omega(\sqrt{\tau \ln(1/\tau)d/n})$ when $\bS \sim \cD_i^n$. 

The distributions we consider each corresponds to a hypothesis $h \in \cH$. The distribution $\cD_h$ returns each point $x \in \cX$ such that $h(x)=-1$ with probability $1/u-\alpha$. For the remaining points, the distribution returns them with probability $1/u + f(\alpha,d,u)$ such that we get a probability distribution (thus $f(\alpha,d,u) < \alpha$ for $u \ge 2d$). Since the unknown concept/true labeling function is the all-1 function, we have that any hypothesis $h \in \cH$ errs precisely when it returns $-1$. Thus in particular, the best hypothesis under $\cD_h$ is $h$ and that hypothesis has $\er_{\cD_h}(h) = d(1/u-\alpha)$. This is the value we set to $\tau$ by choosing $u$ and $\alpha$ appropriately.

Now consider choosing one of the distributions $\cD_h$ uniformly at random and running $\cA$ on $\bS \sim \cD_h^n$. Since $\cA$ is proper, it has to return a hypothesis in $\cH$. This means that is has to choose $d$ points $x_i$ on which to return $-1$. Now crucially, if a constant fraction of those are chosen such that $h(x_i) = 1$, then $\er_{\cD_h}(h_\bS) = \tau + \Omega(d \alpha)$. Intuitively, since the points with $h(x_i)=-1$ receive the least probability mass under $\cD_h$, and $\cA$ does not know the distribution $\cD_h$, the best strategy for $\cA$ is to output the hypothesis $h_\bS$ returning $-1$ on the $d$ points $x_i$ from which there are fewest copies in the sample $\bS$. We expect to see $n(1/u + f(\alpha)) \leq n/u + \alpha n$ copies of each $x_i$ with $h(x_i)=1$ and we expect to see $n(1/u-\alpha) = n/u-\alpha n$ copies of each point with $h(x_i)=-1$. A simple application of Chebyshev's inequality implies that with constant probability, it holds for at least half the points $x_i$ with $h(x_i)=-1$ that we see $n_i \geq n/u - \alpha n - O(\sqrt{n/u})$ copies of it in $\bS$. Now for the points $x_i$ with $h(x_i)=1$, by anti-concentration, we see no more than $n_i = n/u + \alpha n - \Omega(\sqrt{n\ln(u/d)/u})$ copies with probability roughly $d/u$. We thus expect to see $\Omega(d)$ such points with $n_i= n/u + \alpha n - \Omega(\sqrt{n\ln(u/d)/u})$. If $\sqrt{n\ln(u/d)/u} > c \alpha n$ for a large enough $c>0$, this implies we have fewer copies of these points and $\cA$ will return $-1$ on at least $d/2$ of them. We can thus choose $\alpha = \Theta(\sqrt{\ln(u/d)/(un)}) = \Theta(\sqrt{\tau \ln(1/\tau)/(dn)})$ and conclude $\er_{\cD_h}(h_\bS) = \tau + \Omega(d \alpha) = \tau + \Omega(\sqrt{\tau d\ln(1/\tau)/n})$ as claimed.

\section{Near-Optimal Agnostic PAC Learner}
In this section, we present our new agnostic PAC learner, \textsc{DisagreeingExperts} (Algorithm~\ref{alg:wrap}), with an optimal error bound except for very small values of $\tau = \inf_{h \in \cH} \er_\cD(h)$. The guarantees of Algorithm~\ref{alg:wrap} are stated in our main upper bound result, Theorem~\ref{thm:mainupper}. To simplify the analysis, \textsc{DisagreeingExperts} ensures that we may focus on analysing a subroutine \textsc{CoreDisagreeingExperts} (Algorithm~\ref{alg:agnostic}) under the following simplifying assumptions:
\begin{enumerate}
    \item $n \geq c_n \ln^{3.5}(n/d)(d  + \ln(1/\delta))$ for large enough constant $c_n>0$.
    \item $c_\tau \ln^9(n/d)(d + \ln(1/\delta))/n \leq \tau \leq 1/c_\tau$ for large enough constant $c_\tau>0$.
    \item We have an estimate $\tilde{\tau} \in [\tau/2,2\tau]$ available.
\end{enumerate}
Under these assumptions, we show that the subroutine \textsc{CoreDisagreeingExperts} (Algorithm~\ref{alg:agnostic}) with probability at least $1-\delta$ over a training set $\bS \sim \cD^n$, returns a hypothesis $h_\bS$ with $\er_{\cD}(h_\bS) \leq \tau + O(\sqrt{\tau(d + \ln(1/\delta))/n})$. 

We start by justifying these assumptions before delving into the details of the analysis. \emph{Crucially}, our full algorithm \textsc{DisagreeingExperts} needs \emph{none} of these assumption, they are merely to ease the presentation and analysis of the main part of our algorithm, \textsc{CoreDisagreeingExperts}.

\paragraph{Simplifying assumptions.}
For assumption 1., notice that our claimed upper bound in Theorem~\ref{thm:mainupper} on the error of $h_\bS$ exceeds $1$ for smaller $n$ and thus is trivially true. The algorithm \textsc{DisagreeingExperts}, shown as Algorithm~\ref{alg:wrap}, takes care of assumptions 2.\ and 3.

\begin{algorithm}
  \DontPrintSemicolon
  \KwIn{Training set $S$ of $n$ samples $\{(x_i,y_i)\}_{i=1}^{n}$ with $(x_i,y_i) \in \cX \times \{-1,1\}$, hypothesis set $\cH$ of VC-dimension $d$, failure parameter $\delta>0$.
  }
  \KwResult{Classifier $h_S : \cX \to \{-1,1\}$.}
  Partition $S$ into three sets $S_1,S_2, S_3$ of $n/3$ samples each.
  
    Let $\tilde{\tau} \gets \er_{S_1}(h^\star_{S_1})$.

    Run \textsc{CoreDisagreeingExperts}($S_2,\cH,d,\delta,\tilde{\tau}$) to obtain hypothesis $h_1$.

    Run ERM on $S_2$ to obtain hypothesis $h_2$.

    \Return{$h_S \in \{h_1,h_2\}$ with smallest $\er_{S_3}(h_S)$.}
    
  \caption{\textsc{DisagreeingExperts}($S,\cH,d,\delta$)}\label{alg:wrap}
\end{algorithm}

Given a training set $\bS \sim \cD^n$, \textsc{DisagreeingExperts} first splits the training set into $3$ sets $\bS_1,\bS_2,\bS_3$ of $n/3$ samples each. It then computes the error $\tilde{\tau}$ of the best hypothesis $h^\star_{\bS_1}$ in $\cH$ on $\bS_1$. By the ERM Theorem (Theorem~\ref{thm:ERM}) and assumption 1.\ (that we already justified), the estimate $\tilde{\tau}$ satisfies $\tilde{\tau}  \in [\tau/2,2 \tau]$ with probability $1-\delta$. 

It then invokes \textsc{CoreDisagreeingExperts} on $\bS_2$ using this estimate $\tilde{\tau}$ to obtain a hypothesis $\bh_1$. This justifies assumption 3.\ (by rescaling $\delta$ by a constant factor).

It also runs ERM on $\bS_2$ to obtain a hypothesis $\bh_2$. Finally, it uses $\bS_3$ as a validation set to estimate $\er_\cD(\bh_1)$ and $\er_{\cD}(\bh_2)$ to within additive (by Chernoff):
\[
|\er_{\bS_3}(\bh_i) - \er_{\cD}(\bh_i)| = O(\sqrt{\er_\cD(\bh_i) \ln(1/\delta)/n} + \ln(1/\delta)/n).
\]
Returning the hypothesis among $\bh_1,\bh_2$ with the least $\er_{\bS_3}(\bh_i)$ ensures that the final hypothesis $h_\bS$ has error at most
\[
\min_i \er_{\cD}(\bh_i) + O(\sqrt{\er_\cD(\bh_i) \ln(1/\delta)/n} + \ln(1/\delta)/n).
\]
By the guarantee claimed above for Algorithm~\ref{alg:agnostic}, this is at most $\tau + O(\sqrt{\tau(d + \ln(1/\delta))/n})$ when $c_\tau \ln^9(n/d)(d + \ln(1/\delta))/n \leq \tau \leq 1/c_\tau$ (i.e.\ under assumption 2.). For smaller $\tau$, the ERM Theorem (Theorem~\ref{thm:ERM}) guarantees that $\bh_2$ has an error of at most 
\[
\tau + O\left(\sqrt{\frac{\tau(d \ln(n/d) + \ln(1/\delta))}{n}} + \frac{d \ln(n/d) + \ln(1/\delta)}{n}\right) = \tau + O\left( \frac{\ln^{5}(n/d)(d  + \ln(1/\delta))}{n} \right).
\]
Note that we have upper bounded $\tau \ln(1/\tau)$ in the first term by $\tau \ln(n/d)$ since the second term dominates for $\tau \ll d/n$. Similarly, for $\tau > 1/c_\tau$, the ERM Theorem guarantees that $\bh_2$ has an error of at most
\[
\tau + O\left(\sqrt{(d + \ln(1/\delta))/n} \right) = \tau  + O\left(\sqrt{\tau(d + \ln(1/\delta))/n} \right).
\]
This completes the justifications for assumptions 1., 2., and 3. We now proceed to analyzing the main part of our new algorithm, denoted \textsc{CoreDisagreeingExperts}, under these assumptions.

\subsection{Core algorithm}
Our algorithm \textsc{CoreDisagreeingExperts} is shown as Algorithm~\ref{alg:agnostic}, where we define
\[
\alpha(n,d,\delta,\beta):=c_\alpha \left(\sqrt{\frac{\beta(d \ln(1/\beta) + \ln(1/\delta))}{n}} + \frac{d\ln(n/d) + \ln(1/\delta)}{n}\right)
\]
with $c_\alpha$ a sufficiently large constant. In Algorithm~\ref{alg:agnostic}, the two parameters $c_t,c_Z$ are also sufficiently large constants (in particular, $c_Z$ is sufficiently larger than $c_\alpha + c_t$ and $c_\alpha$ is sufficiently larger than the constant hiding in the $O(\cdot)$-notation of the ERM Theorem (Theorem~\ref{thm:ERM})).

\begin{algorithm}
  \DontPrintSemicolon
  \KwIn{Training set $S$ of $2n$ samples $\{(x_i,y_i)\}_{i=1}^{2n}$ with $(x_i,y_i) \in \cX \times \{-1,1\}$, hypothesis set $\cH$ of VC-dimension $d$, failure parameter $\delta>0$, estimate $\tilde{\tau} \in [\tau/2,2\tau]$.
  }
  \KwResult{Classifier $h_S : \cX \to \{-1,1\}$.}
  Partition $S$ into two sets $B,C$ of $n$ samples each.
  
    Let $t \gets c_t \ln(1/\tilde{\tau}) \ln \ln(1/\tilde{\tau})$.

    Let $Z_t \gets c_Z \cdot t \ln^2(n/d)\left(d \ln(n/d) + \ln(1/\delta)\right)/n$.

    Partition $B$ into $t$ sets $B^1,\dots,B^t$ of $n/t$ samples each.

    $r \gets 0$

    \For{$i=1,\dots,t$}{
        Let $T^i \subseteq B^i$ be the samples in $B^i$ with $h^j_1(x)=h^j_2(x)$ for all $j < i$.

        Run ERM on $T^i$ to obtain a hypothesis $h^\star_{T^i}$.

        Let $\gamma_i \gets \er_{T^i}(h^\star_{T^i})$.

        \If{$\gamma_i \leq Z_t$}{
           \textbf{break}
        }

        Let $\cH^i \subseteq \cH$ be the hypotheses $h \in \cH$ with $\er_{T^i}(h) \leq \gamma_i + \alpha(n/t,d,\delta,\gamma_i)$.

        \If{there is no pair $h_1,h_2 \in \cH^i$ with $\Pr_{
      T^i}[h_1(\bx) \neq h_2(\bx)] \geq
    \gamma_i/\ln(1/\gamma_i)$}{
        \textbf{break}
    }\Else{
        Let $h^i_1\gets h_1$ and $h^i_2 \gets h_2$ for a pair $h_1,h_2 \in \cH^i$ with $\Pr_{
      T^i}[h_1(\bx) \neq h_2(\bx)] \geq
    \gamma_i/\ln(1/\gamma_i)$.

        $r \gets i$
    }
    }
    Partition $C$ into two sets $C_=$ and $C_{\neq}$ where $C_=$ contains all $x$ with $h_1^i(x)=h_2^i(x)$ for all $i=1,\dots,r$ and $C_{\neq}$ contains the remaining.

    Run ERM on $C_=$ to obtain a hypothesis $h^\star_{C_=}$.

    Run ERM on $C_{\neq}$ to obtain a hypothesis $h^\star_{C_{\neq}}$.

    Let $h_S$ be the classifier that on an input $x$ checks whether $h_1^i(x)=h_2^i(x)$ for all $i=1,\dots,r$. If so, $h_S$ returns $h^\star_{C_=}(x)$ and otherwise it returns $h^\star_{C_{\neq}}(x)$.
    
    \Return{$h_S$.}
  \caption{\textsc{CoreDisagreeingExperts}($S,\cH,d,\delta, \tilde{\tau}$)}\label{alg:agnostic}
\end{algorithm}

Recall that our goal is to show that under assumptions 1., 2.\ and 3., it holds with probability at least $1-\delta$ over a training set $\bS \sim \cD^{2n}$, that Algorithm~\ref{alg:agnostic} returns a hypothesis $h_\bS$ with $\er_{\cD}(h_\bS) \leq \tau + O(\sqrt{\tau(d + \ln(1/\delta))/n})$. We note that the algorithm is presented as if given a training set of size $2n$, not $n$. This is merely to make the constants simpler and only affects the generalization error by a constant factor after rescaling $n$ with  $n/2$.

\paragraph{Brief overview.}
Before giving the details of the analysis, let us briefly discuss the steps of Algorithm~\ref{alg:agnostic}, introduce some notation and give the high level ideas in the analysis. Assume we run Algorithm~\ref{alg:agnostic} on a data set $\bS \sim \cD^{2n}$. The data set is first split into two pieces $\bB,\bC$ of size $n$ each.

We start by partitioning $\bB$ into $t$ pieces $\bB^1,\dots,\bB^t$ of $n/t$ samples each and execute the for-loop in steps 6-17. The goal of these steps is to obtain hypotheses $\bh^i_1$ and $\bh^i_2$ that are both close to optimal and yet disagree a lot in their predictions, i.e., disagreeing experts. In each step of the loop, we gather the set $\bT^i$ of samples $(\bx,\by) \in \bB^i$ for which $\bh^j_1(\bx)=\bh^j_2(\bx)$ for all $j<i$, i.e., none of the previous pairs disagree on $\bx$ (for $i=1$, we have $\bT^1=\bB^1$). Now consider any fixed outcome $B^1,\dots,B^{i-1}$ of $\bB^1,\dots,\bB^{i-1}$ and $h^1_1,h^1_2,\dots,h^{i-1}_1,h^{i-1}_2$ of $\bh_1^1,\bh^1_2,\dots,\bh^{i-1}_1,\bh^{i-1}_2$. The samples in $\bT^i$ are i.i.d.\ from $\cD$ conditioned on $h^j_1(\bx)=h^j_2(\bx)$ for all $j<i$. Denote this conditional distribution by $\cD^i$. Steps 8-9 estimate the best possible error $\er_{\cD^i}(h^\star_{\cD^i})$ achievable under $\cD^i$. If this error is sufficiently small, we exit the for-loop in step 11.

If not, we gather the subset of hypotheses $\bH^i$ that are near-optimal on the data set $\bT^i$ (step 12). Among these, we look for a pair $h_1,h_2$ that disagree on many predictions in $\bT^i$. If there is no such pair, we exit the for-loop in step 14. Finally, if there is, we let $\bh^i_1$ and $\bh^i_2$ be an arbitrary such pair.

Once the for-loop has completed, we use the obtained pairs $\bh^i_1,\bh^i_2$ to partition the samples in $\bC$ into two sets $\bC_=$ and $\bC_{\neq}$, where $\bC_{\neq}$ contains the samples $(x,y) \in \bC$ where at least one $i$ has $\bh^i_1(x)\neq \bh^i_2(x)$ and $\bC_=$ contains the remaining. We finally run ERM on each of the two sets to obtain hypotheses $h^\star_{\bC_=}$ and $h^\star_{\bC_{\neq}}$.

The intuition for why the above works was also discussed in Section~\ref{sec:overview}. We repeat the main ideas here in context of the full algorithm description. First, if we exit the for-loop before having completed all $t$ steps, then either it was possible to obtain a very small error on $\bT^i$ (step 10-11) or there was no pair $h^i_1,h^i_2$ that disagree on many predictions (step 13-14). In the former case, ERM on $\bC_=$ ensures that $h^\star_{\bC_=}$ makes few mistakes on samples from $\cD$ where $\bh_1^i(\bx)=\bh_2^i(\bx)$ for all $i$. Denote the distribution of such a sample by $\bD_{=}$. In the latter case, since all hypotheses that are near-optimal on $\bT^i$ make almost the same predictions, the ERM bounds improve for $\bC_=$. If we complete all $t$ iterations of the for-loop, then we will show that each step decreases $\er_{\bD^i}(h^\star_{\bD^i})$ enough that $\er_{\bD_=}(h^\star_{\bD_=}) \leq \tau/\ln(1/\tau)$. With this reduced error, the additive mistakes resulting from ERM is down-scaled sufficiently to cancel out the $\sqrt{\ln(1/\tau)}$ factor of sub-optimality. 

Finally, for the set $\bC_{\neq}$, we will show that $\Pr_{\cD}[\exists i : \bh_1^i(\bx) \neq \bh_2^i(\bx)] \leq O(\tau)$. Thus, when we run ERM on $\bC_{\neq}$, we can afford to merely upper bound the error of $h^\star_{\bC_{\neq}}$ by $\er_{\bD_{\neq}}(h^\star_{\bD_{\neq}}) + O(\sqrt{(d + \ln(1/\delta))/(\tau n)})$. This is because we only see such a sample with probability $O(\tau)$ and thus the additive error contributes only $O(\tau \sqrt{(d + \ln(1/\delta))/(\tau n)}) = O(\sqrt{\tau(d + \ln(1/\delta))/n})$. Here $\bD_{\neq}$ denotes the conditional distribution of a sample $(\bx,\by)$ from $\cD$ conditioned on there being at least one $i$ for which $\bh^i_1(\bx) \neq \bh^i_2(\bx)$.

\paragraph{Analysis.}
We are now ready for the formal correctness proof. Let $\bS \sim \cD^{2n}$ denote a random training set of $2n$ samples and let $\bB$ and $\bC$ be the respective sets of size $n$ constructed by Algorithm~\ref{alg:agnostic}.

We first argue that once the for-loop in steps 6-17 of Algorithm~\ref{alg:agnostic} terminates, the two distributions $\bD_=$ and $\bD_{\neq}$ have the following desirable properties:
\begin{lemma}
\label{lem:goodfor}
    It holds with probability at least $1-\delta/2$ over $\bB \sim \cD^{n}$, that upon termination of the for-loop, we have $\Pr_{\cD}[\exists i : \bh_1^i(\bx) \neq \bh_2^i(\bx)] \leq 8\tau$ and:
    \begin{itemize}
        \item For any $m \geq n/2$, it holds with probability at least $1-\delta/8$ over a set $\bC_{=} \sim \bD_=^m$ that $\er_{\bD_=}(h^\star_{\bC_=}) = er_{\bD_=}(h^\star_{\bD_=}) + O(\sqrt{\tau(d + \ln(1/\delta))/n})$.
    \end{itemize}
\end{lemma}
Before proving Lemma~\ref{lem:goodfor}, we show that it suffices to establish our claim on $\er_{\cD}(h_\bS)$ and thus completes the proof of Theorem~\ref{thm:mainupper}. 

Fix an arbitrary outcome $B$ of $\bB$ for which the properties in Lemma~\ref{lem:goodfor} are satisfied upon termination. This also fixed $\bh_1^i,\bh_2^i$ to some $h_1^i,h_2^i$ and $\bD_{=}$ and $\bD_{\neq}$ to some $\cD_{=}$ and $\cD_{\neq}$. The set $\bC$ still consists of i.i.d.\ samples from $\cD$ as $\bC$ is not used in the for-loop.

Define $p:= \Pr_{\cD}[\exists i : h_1^i(\bx) \neq h_2^i(\bx)]$. We start by showing properties of $h^\star_{\bC_{\neq}}$ when $p \geq c_p \ln(1/\delta)/n$ for a large enough constant $c_p$.  Under this assumption on $p$, by Chernoff, we have $|\bC_{\neq}| \geq (p/2)n$ except with probability $1-\delta/8$. In this case, it follows from the ERM Theorem (Theorem~\ref{thm:ERM}) that with probability $1-\delta/8$, $h^\star_{\bC_{\neq}}$ has 
\[
\er_{\cD_{\neq}}(h^\star_{\bC_{\neq}}) = \er_{\cD_{\neq}}(h^\star_{\cD_{\neq}}) + O(\sqrt{(d+\ln(1/\delta))/(p n)}).
\]
Next, we show properties of $h^\star_{\bC_{=}}$. Since we assume $\tau \leq 1/c_\tau$ for a big enough constant $c_\tau$, we must have $(1-p)\geq 1-8 \tau \geq 3/4$. Since $n$ is assumed sufficiently large, this implies that with probability at least $1-\delta/8$, we have $|\bC_=| \geq n/2$. Conditioned on this, by the properties in Lemma~\ref{lem:goodfor}, we have with probability at least $1-\delta/8$ that
\begin{eqnarray*}
\er_{\cD_{=}}(h^\star_{\bC_{=}}) = \er_{\cD_{=}}(h^\star_{\cD_{=}}) +  O\left(\sqrt{\frac{\tau (d + \ln(1/\delta))}{n}}\right).
\end{eqnarray*}
The returned hypothesis $h_\bS$ thus satisfies
\begin{eqnarray*}
    \er_{\cD}(h_\bS) &=& p \er_{\cD_{\neq}}(h^\star_{\bC_{\neq}}) + (1-p)\er_{\cD_=}(h^\star_{\bC_=}).
\end{eqnarray*}
If $p < c_p \ln(1/\delta)/n$ (and by assumption that $\tau \geq c_\tau \ln^9(n/d)(d + \ln(1/\delta))/n$), this gives
\begin{eqnarray*}
    \er_{\cD}(h_\bS) &\leq& p  + (1-p)\er_{\cD_=}(h^\star_{\bC_=}) \\
    &=& p + (1-p)\er_{\cD_=}(h^\star_{\cD_=}) + O\left(\sqrt{\frac{\tau (d + \ln(1/\delta))}{n}}\right) \\
    &=& (1-p)\er_{\cD_=}(h^\star_{\cD}) + O\left(\sqrt{\frac{\tau (d + \ln(1/\delta))}{n}}\right) \\
    &=& \er_{\cD}(h^\star_{\cD}) + O\left(\sqrt{\frac{\tau (d + \ln(1/\delta))}{n}}\right).
\end{eqnarray*}
If $p \geq c_p \ln(1/\delta)/n$, we have from Lemma~\ref{lem:goodfor} that $p \leq 8 \tau$ and thus
\begin{eqnarray*}
    \er_{\cD}(h_\bS) &=& p \er_{\cD_{\neq}}(h^\star_{\cD_{\neq}}) + (1-p)\er_{\cD_=}(h^\star_{\cD_=}) + p\cdot O\left(\sqrt{\frac{d + \ln(1/\delta)}{pn}}\right) + O\left(\sqrt{\frac{\tau (d + \ln(1/\delta))}{n}}\right) \\
    &=& p\er_{\cD_{\neq}}(h^\star_{\cD})+ (1-p)\er_{\cD_=}(h^\star_{\cD})+  O\left(\sqrt{\frac{p(d + \ln(1/\delta))}{n}}\right) + O\left(\sqrt{\frac{\tau (d + \ln(1/\delta))}{n}}\right) \\
    &=& \er_{\cD}(h^\star_\cD) + O\left(\sqrt{\frac{\tau (d + \ln(1/\delta))}{n}}\right).
\end{eqnarray*}
This completes the proof of Theorem~\ref{thm:mainupper}. What remains is thus to establish Lemma~\ref{lem:goodfor}, which is the focus of the next subsection.

\subsection{Progress on termination (proof of Lemma~\ref{lem:goodfor})}
In this section, we prove Lemma~\ref{lem:goodfor}. Intuitively, termination of the for-loop results in the properties claimed in Lemma~\ref{lem:goodfor} provided that the estimates based on performance on the $\bT^i$'s are sufficiently accurate. To formalize this, we define a number of natural \emph{failure events} relating to the accuracy of these.

\paragraph{Failure events.}
We define a number of bad events that we argue rarely occur. Let $\bD^i$ be the random variable giving the distribution of a sample $(\bx,\by) \sim \cD$ conditioned on $\forall j<i : \bh^j_1(\bx)=\bh^j_2(\bx)$. Note that this random variable is determined from $\bB^1,\dots,\bB^{i-1}$. The samples  $\bT^i$ are then i.i.d.\ from $\bD^i$. We define two failure events relating to how well $\bT^i$ represents $\bD^i$:
\begin{enumerate}
\item Let $E_{i,0}$ be the event that Algorithm~\ref{alg:agnostic} reaches iteration $i$, none of the events $E_{j,0},E_{j,1}$ occurred for $i<j$ and there is a hypothesis $h \in \cH$ with $|\er_{\bD^i}(h) - \er_{\bT^i}(h)| > (1/32)\alpha(n/t,d,\delta,\min\{\er_{\bD^i}(h),\er_{\bT^i}(h)\})$.

\item Let $E_{i,1}$ be the event that Algorithm~\ref{alg:agnostic} reaches iteration $i$, none of the events $E_{j,0},E_{j,1}$ occurred for $i<j$ and there is a pair of hypotheses $h_1,h_2$ with $|\Pr_{\bD^i}[h_1(\bx)\neq h_2(\bx)] - \Pr_{\bT^i}[h_1(\bx)\neq h_2(\bx)]| > (1/32)\alpha(n/t,d,\delta,\min\{\Pr_{\bD^i}[h_1(\bx)\neq h_2(\bx)], \Pr_{\bT^i}[h_1(\bx) \neq h_2(\bx)]\})$.
\end{enumerate}

We will show that these events are unlikely:
\begin{lemma}
\label{lem:unlikely}
    For all $i$, we have $\Pr[E_{i,0}] \leq \delta/(4 t)$ and $\Pr[E_{i,1}] \leq \delta/(4t)$.
\end{lemma}
We also show that when none of the events occur, the hypotheses and execution of Algorithm~\ref{alg:agnostic} satisfies the following:
\begin{observation}
\label{obs:props}
Assume none of the events $E_{j,0}$ and $E_{j,1}$ occur for $j \leq i$ and that Algorithm~\ref{alg:agnostic} does not terminate before iteration $i$. Then if $\gamma_i = \er_{\bT^i}(h^\star_{\bT^i}) \leq Z_t$, it holds that $\er_{\bD^i}(h^\star_{\bD^i}) \leq 2Z_t$. If $\gamma_i > Z_t$, then each of the following hold:
\begin{itemize}
\item
$
\gamma_i \leq \er_{\bD^{i}}(h^\star_{\bD^{i}})+\er_{\bD^{i}}(h^\star_{\bD^{i}})/\ln(1/\er_{\bD^{i}}(h^\star_{\bD^{i}})) \leq 2\er_{\bD^{i}}(h^\star_{\bD^{i}}).
$

\item Every hypothesis $h$ in $\bH^{i}$ satisfies $\er_{\bD_{i}}(h) \leq  \er_{\bD^i}(h^\star_{\bD^i}) + (1/8)\er_{\bD^i}(h^\star_{\bD^i})/\ln(1/\er_{\bD^i}(h^\star_{\bD^i}))$.

\item Every hypothesis $h \in \cH$ with $\er_{\bD_i}(h) \leq \er_{\bD^i}(h^\star_{\bD^i}) + (1/8)\alpha(n/t,d,\delta,\er_{\bD^i}(h^\star_{\bD^i}))$ is in $\bH^i$.

\item Every pair of hypotheses $h_1,h_2$ with $\Pr_{\bT^i}[h_1(\bx) \neq h_2(\bx)] \geq \gamma_i/\ln(1/\gamma_i)$ satisfy $\Pr_{\bD^i}[h_1(\bx) \neq h_2(\bx)] \geq (1/2)\er_{\bD^i}(h^\star_{\bD^i})/\ln(1/\er_{\bD^i}(h^\star_{\bD^i}))$.

\item Every pair of hypotheses $h_1,h_2$ with $\Pr_{\bT^i}[h_1(\bx) \neq h_2(\bx)] < \gamma_i/\ln(1/\gamma_i)$ satisfies $\Pr_{\bD^i}[h_1(\bx) \neq h_2(\bx)] \leq 4\er_{\bD^i}(h^\star_{\bD^i})/\ln(1/\er_{\bD^i}(h^\star_{\bD^i}))$.
\end{itemize}
\end{observation}
The proof of Lemma~\ref{lem:unlikely} mostly uses standard concentration results for classes with bounded VC-dimension and has thus been deferred to Appendix~\ref{sec:unlikely}. Similarly, the proof of Observation~\ref{obs:props} merely uses the definition of $\alpha$ and $Z_t$ and has thus been deferred to Appendix~\ref{sec:props}.

More interestingly, we show that if none of the events $E_{i,j}$ occur, then the for-loop makes progress towards reducing $\er_{\bD^i}(h^\star_{\bD^i})$ in each iteration:

\begin{lemma}
  \label{lem:progress}
For any integer $1 \leq i \leq t+1$, assume none of the events $E_{j,0}$ and $E_{j,1}$ occurred for $j<i$ and
  that Algorithm~\ref{alg:agnostic} did not terminate with $r < i-1$. Then
  $\er_{\bD^i}(h^\star_{\bD^i}) \leq \er_{\cD}(h^\star_\cD)(1-1/(32
  \ln(1/\er_{\cD}(h^\star_\cD))))^{i-1}$ and $\Pr_{\cD}[\exists j \leq
  i : \bh_1^j(\bx) \neq \bh_2^j(\bx)] \leq 8\left(\er_{\cD}(h^\star_\cD) - \er_{\bD^{i}}(h^\star_{\bD^{i}}) \right)$.
\end{lemma}

We prove Lemma~\ref{lem:progress} in the next subsection and for now focus on completing the proof of Lemma~\ref{lem:goodfor} from Lemma~\ref{lem:unlikely}, Observation~\ref{obs:props} and Lemma~\ref{lem:progress}. We have restated it here for convenience:

\begin{customlem}{\ref{lem:goodfor}}
    It holds with probability at least $1-\delta/2$ over $\bB \sim \cD^{n}$, that upon termination of the for-loop, we have $\Pr_{\cD}[\exists i : \bh_1^i(\bx) \neq \bh_2^i(\bx)] \leq 8\tau$ and:
    \begin{itemize}
        \item For any $m \geq n/2$, it holds with probability at least $1-\delta/8$ over a set $\bC_{=} \sim \bD_=^m$ that $\er_{\bD_=}(h^\star_{\bC_=}) = er_{\bD_=}(h^\star_{\bD_=}) + O(\sqrt{\tau(d + \ln(1/\delta))/n})$.
    \end{itemize}
\end{customlem}

\begin{proof}[Proof of Lemma~\ref{lem:goodfor}]
From Lemma~\ref{lem:unlikely} and a union bound, we conclude that with probability at least $1-\delta/2$, none of the events $E_{i,j}$ occur. We show that conditioned on this, the properties claimed in Lemma~\ref{lem:goodfor} hold. So fix an outcome $B$ of $\bB$ where the events did not occur. This also fixes $\bh_1^i,\bh_2^i$, $\bT^i$, $\bH^i$, $\bD_i$, $\bD_{=}$ and $\bD_{\neq}$ to some $h_1^i,h_2^i$, $T^i$, $\cH^i$, $\cD^i$, $\cD_=$ and $\cD_{\neq}$. Let $r$ denote the value of the variable $r$ in Algorithm~\ref{alg:agnostic} upon termination.

First, recall that upon termination, we have $\cD_{=} \simeq \cD^{r+1}$. Since none of the events $E_{i,j}$ occurred for any $i \leq t$, we have from Lemma~\ref{lem:progress} (with $i=r+1$) that $\Pr_{\cD}[\exists j \leq r : h^j_1(\bx) \neq h^j_2(\bx)] \leq 8(\er_{\cD}(h^\star_\cD) - \er_{\cD^{r+1}}(h^\star_{\cD^{r+1}})) \leq 8 \tau$. This establishes the first claim in Lemma~\ref{lem:goodfor}.

For the second claim, we split the proof in several cases depending on how the for-loop in Algorithm~\ref{alg:agnostic} terminates on $B$. The main observations, which we will expand upon below, are: 1. if we terminate in step 11, then $\er_{\cD_=}(h^\star_{\cD_=}) = O(Z_t) = O(\sqrt{\tau(d + \ln(1/\delta))/n})$, 2. if we terminate by completing all iterations of the for-loop, then $\er_{\cD_=}(h^\star_{\cD_=}) \leq \tau/\ln(1/\tau)$. Both of these are sufficient to show that the ERM Theorem (Theorem~\ref{thm:ERM}) on $\bC_{\neq}$ is good enough. Finally, if we terminate in step 14, we carefully exploit that all near-optimal hypotheses agree on most samples. This allows for a better guarantee on ERM than invoking the ERM Theorem.

The easiest cases are termination in step 11 and completion of the for-loop, so we argue for those first.

\paragraph{Termination in Step 11.}
Since we terminate in step 11, we must have $\gamma_{r+1} \leq Z_t$.
By Observation~\ref{obs:props}, this implies $\er_{\cD_=}(h^\star_{\cD_=}) = \er_{\cD^{r+1}}(h^\star_{\cD^{r+1}}) \leq 2Z_t$. 

Since $\tilde{\tau} \in [\tau/2,2\tau]$, we have $t = O(\ln(1/\tau)\ln \ln(1/\tau))$ and since we assume $\tau > d/n$, we have $Z_t = O(\ln(n/d)\ln \ln(n/d) \ln^3(n/d)(d + \ln(1/\delta))/n)$. By the ERM Theorem (Theorem~\ref{thm:ERM}) and since we assume $n$ sufficiently large, we have that for any $m \geq n/2$, it holds with probability at least $1-\delta/8$ over a set $\bC_{=} \sim \cD_{=}^m$ that $\er_{\cD_=}(h^\star_{\bC_=}) \leq 4Z_t$. Since we assume $\tau \geq c_\tau \ln^9(n/d)(d + \ln(1/\delta))/n$ for large enough $c_\tau$, we have $\sqrt{\tau(d + \ln(1/\delta))/n} \geq 4Z_t$ and thus $\er_{\cD_=}(h^\star_{\bC_=}) \leq \sqrt{\tau(d + \ln(1/\delta))/n}$ with probability at least $1-\delta/8$ over $\bC_{=}$.

\paragraph{Termination by completion.}
Since none of the events $E_{i,j}$ occurred and we terminate upon completing the for-loop, we have $r \gets t$ and we get from Lemma~\ref{lem:progress} (with $i=t+1=r+1$) that $\er_{\cD_=}(h^\star_{\cD_=}) = \er_{\cD^{r+1}}(h^\star_{cD^{i+1}}) \leq \tau(1-1/(32 \ln(1/\tau)))^t$. This is at most $\tau \exp(-t/(32 \ln(1/\tau)))$.
Since $\tilde{\tau} \geq \tau/2$ we have $t \geq 32 \ln(1/\tau)\ln \ln(1/\tau)$ for $c_t$ large enough. Thus $\er_{\cD_=}(h^\star_{\cD_=}) \leq \tau/\ln(1/\tau)$. The ERM Theorem (Theorem~\ref{thm:ERM}) now implies that with probability at least $1-\delta/8$ over a set $\bC_{=} \sim \cD_{=}^m$ with $m \geq n/2$, we have
\[
\er_{\cD_=}(h^\star_{\bC_=}) = \er_{\cD_=}(h^\star_{\cD_=}) + O\left(\sqrt{\frac{\er_{\cD_=}(h^\star_{\cD_=})(d \ln(\frac{1}{\er_{\cD_=}(h^\star_{\cD_=})}) + \ln(1/\delta))}{n}}  + \frac{d \ln(n/d) + \ln(1/\delta)}{n}\right).
\]
For $\er_{\cD_=}(h^\star_{\cD_=}) \leq \tau/\ln(1/\tau)$, we have $\er_{\cD_=}(h^\star_{\cD_=}) \ln(1/\er_{\cD_=}(h^\star_{\cD_=})) = O(\tau)$ and thus since we assume $\tau \geq c_\tau \ln^9(n/d)(d + \ln(1/\delta))$, we conclude
\[
\er_{\cD_=}(h^\star_{\bC_=}) = \er_{\cD_=}(h^\star_{\cD_=}) + O\left(\sqrt{\frac{\tau(d + \ln(1/\delta))}{n}} \right).
\]
\paragraph{Termination in Step 14.}
Assume we terminate in step 14 of some iteration $i$ and let $r \gets i-1$. Then by definition of Algorithm~\ref{alg:agnostic}, there is no pair $h_1, h_2 \in \cH^i$ with $\Pr_{T^i}[h_1(\bx) \neq h_2(\bx)] \geq \gamma_i/\ln(1/\gamma_i)$ with $\gamma_i = \er_{T^i}(h^\star_{T^i})$. 

Now define $\bar{\cH}^i \subseteq \cH$ as the set of all hypotheses $h \in \cH$ with $\er_{\cD^i}(h) \leq \er_{\cD^i}(h^\star_{\cD^i}) + (1/8)\alpha(n/t,d,\delta,\er_{\cD^i}(h^\star_{\cD^i}))$. By Observation~\ref{obs:props}, all pairs $h_1,h_2 \in \bar{\cH}^i$ are in $\cH^i$ and thus have $\Pr_{T_i}[h_1(\bx) \neq h_2(\bx)] < \gamma_i/\ln(1/\gamma_i)$. From Observation~\ref{obs:props}, this further implies that $\Pr_{\cD^i}[h_1(\bx) \neq h_2(\bx)] \leq 4 \er_{\cD^i}(h^\star_{\cD^i})/\ln(1/\er_{\cD^i}(h^\star_{\cD^i}))$.

Consider now a set $\bC_{=} \sim \cD_{=}^m$ for an $m \geq n/2$ and recall $\cD_{=} \simeq \cD_i$ when we terminate in iteration $i$ of the for-loop. By the ERM Theorem (Theorem~\ref{thm:ERM}) and for the constant $c_\alpha$ in the definition of $\alpha$ large enough, we have that with probability at least $1-\delta/24$, all hypotheses $h \in \cH \setminus \bar{\cH}^i$ have $\er_{\bC_{=}}(h) \geq \er_{\cD_{=}}(h^\star_{\cD_=}) + (1/16)\alpha(n/t,d,\delta,\er_{\cD^i}(h^\star_{\cD^i}))$. 

Finally, for the hypotheses in $\bar{\cH}^i$, we have by definition that any $h \in \bar{\cH}^i$ has $\Pr_{\cD_{=}}[h(\bx) \neq h^\star_{\cD_{=}}(\bx)] \leq 4 \er_{\cD_=}(h^\star_{\cD_=})/\ln(1/\er_{\cD_=}(h^\star_{\cD_=}))$. We now invoke the following improved version of the ERM Theorem for hypothesis sets with such properties:
\begin{lemma}
\label{lem:betteruni}
  Let $\cD$ be a distribution over $\cX$ and $\cH \subset \cX \to \{-1,1\}$ a hypothesis set of VC-dimension $d$. Assume there is a hypothesis $h_0 \in \cH$ such that for all $h \in \cH$, we have $\Pr_\cD[h(\bx) \neq h_0(\bx)] \leq p$. Then for any $0 < \delta < 1$, it holds with probability $1-\delta$ over a set $\bS$ of $n$ i.i.d.\ samples from $\cD$ that
  \begin{eqnarray*}
    \sup_{h \in \cH} \left|\er_{\bS}(h) - \er_\cD(h) \right| = 
    \left| \er_{\bS}(h_0) - \er_\cD(h_0)\right| + O\left(\sqrt{\frac{p (\ln(1/p) d + \ln(1/\delta))}{n}} + \frac{d\ln(n/d) + \ln(1/\delta)}{n} \right).
  \end{eqnarray*}
\end{lemma}
Applying Lemma~\ref{lem:betteruni} on $\bar{\cH}^i$ with $p = 4 \er_{\cD_=}(h^\star_{\cD_=})/\ln(1/\er_{\cD_=}(h^\star_{\cD_=}))$ and $h_0 = h^\star_{\cD_=}$ gives with probability at least $1-\delta/24$ over $\bC_{=}$ that
\begin{align*}
    &\sup_{h \in \bar{\cH}^i} \left|\er_{\bC_=}(h) - \er_{\cD_=}(h) \right| =
    \\
    &\qquad\left| \er_{\bC_=}(h^\star_{\cD_=}) - \er_{\cD_=}(h^\star_{\cD_=})\right| + O\left(\sqrt{\er_{\cD_=}(h^\star_{\cD_=}) (d + \ln(1/\delta))/n} + (d \ln(n/d) + \ln(1/\delta))/n \right).
  \end{align*}
At the same time, for the fixed hypothesis $h^\star_{\cD_=}$, we have with probability at least $1-\delta/24$ (by Chernoff) that 
\[
\left| \er_{\bC_=}(h^\star_{\cD_=}) - \er_{\cD_=}(h^\star_{\cD_=})\right| = O\left(\sqrt{\er_{\cD_=}(h^\star_{\cD_=}) \ln(1/\delta)/n} + \ln(1/\delta)/n\right).
\]
It follows that ERM on $\bC_{=}$ will return a hypothesis $h^\star_{\bC_=}$ from $\bar{\cH}^i$ and that hypothesis has
\begin{eqnarray}
\label{eq:almost}
\er_{\bC_=}(h^\star_{\bC_=}) = \er_{\cD_=}(h^\star_{\cD_=}) + O\left(\sqrt{\er_{\cD_=}(h^\star_{\cD_=}) (d + \ln(1/\delta))/n} + (d \ln(n/d) + \ln(1/\delta))/n \right).
\end{eqnarray}
Finally, from Lemma~\ref{lem:progress}, we have $\Pr_{\cD}[\exists j \leq r : h_1^j(\bx) \neq h_2^j(\bx)] \leq 8 \tau$. Since $\tau \leq 1/c_\tau$ for large enough $c_\tau$, this implies $\Pr_{\cD}[\forall j \leq r : h_1^j(\bx) = h_2^j(\bx)] \geq 1/2$ and thus it must be the case that $\er_{\cD_=}(h^\star_{\cD_=}) \leq 2 \tau$. Inserting this in~\eqref{eq:almost} and using the assumption $\tau \geq c_\tau \ln^9(n/d)(d + \ln(1/\delta))/n$ yields
\[
\er_{\bC_=}(h^\star_{\bC_=}) = \er_{\cD_=}(h^\star_{\cD_=}) + O\left(\sqrt{\tau (d + \ln(1/\delta))/n}\right).
\]
This completes the proof of Lemma~\ref{lem:goodfor} subject to proving Lemma~\ref{lem:unlikely}, Observation~\ref{obs:props}, Lemma~\ref{lem:progress} and Lemma~\ref{lem:betteruni}. As mentioned earlier, we prove Lemma~\ref{lem:unlikely} and Observation~\ref{obs:props} in Appendix~\ref{sec:unlikely} and Appendix~\ref{sec:props}. We prove Lemma~\ref{lem:progress} in Section~\ref{sec:progress} and Lemma~\ref{lem:betteruni} in Section~\ref{sec:betteruni}.
\end{proof}

\subsection{Progress in For-Loop (proof of Lemma~\ref{lem:progress})}
\label{sec:progress}
In this section, we prove Lemma~\ref{lem:progress} stating that each iteration of the for-loop reduces $\er_{\bD^i}(h^\star_{\bD^i})$ while only increasing $\Pr_{\cD}[\exists j \leq
  i : \bh_1^j(\bx) \neq \bh_2^j(\bx)]$ slightly. We have restated Lemma~\ref{lem:progress} here for convenience:

\begin{customlem}{\ref{lem:progress}}
For any integer $1 \leq i \leq t+1$, assume none of the events $E_{j,0}$ and $E_{j,1}$ occurred for $j<i$ and
  that Algorithm~\ref{alg:agnostic} did not terminate with $r < i-1$. Then
  $\er_{\bD^i}(h^\star_{\bD^i}) \leq \er_{\cD}(h^\star_\cD)(1-1/(32
  \ln(1/\er_{\cD}(h^\star_\cD))))^{i-1}$ and $\Pr_{\cD}[\exists j \leq
  i : \bh_1^j(\bx) \neq \bh_2^j(\bx)] \leq 8\left(\er_{\cD}(h^\star_\cD) - \er_{\bD^{i}}(h^\star_{\bD^{i}}) \right)$.
\end{customlem}

The main idea behind the proof is to consider the pair of hypotheses $\bh_1^{i-1}$ and $\bh_2^{i-1}$. These two hypotheses have a near-optimal error under $\bD^{i-1}$ and yet disagree on the classification of many points. Since one of them is incorrect when they disagree, this intuitively implies that they have to err significantly less when they agree. Observing that $\bD^{i}$ is the distribution $\cD$ conditioned on $\bh^{j}_1(\bx)=\bh^{j}_2(\bx)$ for all $j < i$, this implies that $\er_{\bD^i}(h^\star_{\bD^i})$ is smaller than $\er_{\bD^{i-1}}(h^\star_{\bD^{i-1}})$ and thus we have made progress. We formalize this intuition in the following proof.

\begin{proof}[Proof of Lemma~\ref{lem:progress}]
We prove the lemma by induction in $i$. In the base case $i=1$, we have $\cD^1 = \cD$
  and $\er_{\cD^i}(h^\star_{\cD^i}) = \er_\cD(h^\star_\cD) = \tau$ and the
  claim clearly holds.

For the inductive step, consider any fixed outcome $B^1,\dots,B^{i-1}$ of $\bB^1,\dots,\bB^{i-1}$ for which the algorithm did not terminate with $r < i-1$ and where none of the events $E_{j,0}, E_{j,1}$ occurred for $j < i$. This also fixes an outcome $h_1^j,h_2^j, T^j$, $\cH^j$ of $\bh_1^j,\bh_2^j, \bT^j$, $\bH^j$ for $j<i$ and an outcome $\cD^j$ of $\bD^j$ for $j \leq i$.

Since the algorithm did not terminate with $r < i-1$, in iteration
$i-1$, there was a pair $h_1,h_2 \in \cH^{i-1}$ with $\Pr_{
  T^{i-1}}[h_1(\bx) \neq h_2(\bx)] \geq \gamma_{i-1}/\ln(1/\gamma_{i-1})$
and we have $h^{i-1}_1=h_1$ and $h^{i-1}_2=h_2$ for some such
pair. Furthermore, the events $E_{j,0}, E_{j,1}$ did not occur for any $j < i$. Thus by Observation~\ref{obs:props}, both
$h_1$ and $h_2$ satisfy $\er_{\cD^{i-1}}(h_j) \leq
\er_{\cD^{i-1}}(h^\star_{\cD^{i-1}}) + (1/8)\er_{\cD^{i-1}}(h^\star_{\cD^{i-1}})/\ln(1/\er_{\cD^{i-1}}(h^\star_{\cD^{i-1}})) \leq 2\er_{\cD^{i-1}}(h^\star_{\cD^{i-1}})$. Also, from Observation~\ref{obs:props}, we have 
\[
\Pr_{\cD^{i-1}}[h_1(\bx)
\neq h_2(\bx)] \geq (1/2)\er_{\cD^{i-1}}(h^\star_{\cD^{i-1}})/\ln(1/\er_{{\cD^{^i-1}}}(h^\star_{\cD^{^i-1}})).
\]
We now have 
\begin{eqnarray}
\label{eq:upagree}
    \er_{\cD^i}(h^\star_{\cD^i}) \leq \er_{\cD^i}(h_1) =
(1/2)(\er_{\cD^i}(h_1) + \er_{\cD^i}(h_2) )
\end{eqnarray}
as $h_1$ and $h_2$ agree under $\cD^i$. We see that
\begin{eqnarray*}
    2\er_{\cD^{i-1}}(h^\star_{\cD^{i-1}}) + (1/4)\er_{\cD^{i-1}}(h^\star_{\cD^{i-1}})/\ln(1/\er_{\cD^{i-1}}(h^\star_{\cD^{i-1}})) &\geq&
    \er_{\cD^{i-1}}(h_1) + \er_{\cD^{i-1}}(h_2).
\end{eqnarray*}
Using that precisely one of $h_1$ and $h_2$ errs when they disagree, and that the distribution $\cD^{i-1}$ conditioned on $h_1(\bx)=h_2(\bx)$ is the distribution $\cD^i$ we get 
\begin{align*}
    &\er_{\cD^{i-1}}(h_1) + \er_{\cD^{i-1}}(h_2) \\
    &=
    \Pr_{\cD^{i-1}}[h_1(\bx)=h_2(\bx)](\er_{\cD^{i}}(h_1) +
  \er_{\cD^{i}}(h_2)) + \Pr_{\cD^{i-1}}[h_1(\bx)\neq h_2(\bx)].
\end{align*}
Hence
\begin{align*}
  &\er_{\cD^i}(h_1) + \er_{\cD^i}(h_2) \\
  &\leq \Pr_{\cD^{i-1}}[h_1(\bx)=h_2(\bx)]^{-1}\left(2\er_{\cD^{i-1}}(h^\star_{\cD^{i-1}}) + \frac{1}{4} \cdot \frac{\er_{\cD^{i-1}}(h^\star_{\cD^{i-1}})}{\ln(1/\er_{\cD^{i-1}}(h^\star_{\cD^{i-1}}))}-\Pr_{\cD^{i-1}}[h_1(\bx)\neq
  h_2(\bx)]  \right) \\
  &\leq \Pr_{\cD^{i-1}}[h_1(\bx)=h_2(\bx)]^{-1}\left(2\er_{\cD^{i-1}}(h^\star_{\cD^{i-1}}) -\Pr_{\cD^{i-1}}[h_1(\bx)\neq
  h_2(\bx)]/2  \right).
\end{align*}
At the same time, we have, for $\Pr_{\cD^{i-1}}[h_1(\bx) \neq h_2(\bx)]
\leq 1/2$, that $\Pr_{\cD^{i-1}}[h_1(\bx)=h_2(\bx)]^{-1} \leq 1 + 2
\Pr_{\cD^{i-1}}[h_1(\bx) \neq h_2(\bx)]$.

To see that $\Pr_{\cD^{i-1}}[h_1(\bx) \neq h_2(\bx)] \leq 1/2$, we first get from the induction hypothesis that $\Pr_{\cD}[\exists j < i-1 : h_1(\bx) \neq h_2(\bx)] \leq 8 \tau$. Hence for $\tau \leq 1/c_\tau$ for large enough $c_\tau$, we have $\Pr_{\cD}[\forall j < i-1 : h_1(\bx) = h_2(\bx)] \geq 1/2$. This further implies $\er_{\cD^{i-1}}(h^\star_{\cD^{i-1}}) \leq 2 \tau$. But then $\Pr_{\cD^{i-1}}[h_1(\bx) \neq h_2(\bx)] \leq \er_{\cD^{i-1}}(h_1) + \er_{\cD^{i-1}}(h_2) \leq 4 \er_{\cD^{i-1}}(h^\star_{\cD^{i-1}}) \leq 8 \tau \leq 1/2$ as claimed. 

Again, for $\er_{\cD^{i-1}}(h^\star_{\cD^{i-1}}) \leq 2\tau < 2/c_\tau$ for large enough $c_\tau > 0$, we finally conclude
\begin{align*}
  &\er_{\cD^i}(h_1) + \er_{\cD^i}(h_2) \\
  &\leq
  (1 + 2\Pr_{\cD^{i-1}}[h_1(\bx) \neq h_2(\bx)])\left(2\er_{\cD^{i-1}}(h^\star_{\cD^{i-1}}) -\Pr_{\cD^{i-1}}[h_1(x)\neq
  h_2(x)]/2  \right) \\
  &\leq
  2\er_{\cD^{i-1}}(h^\star_{\cD^{i-1}}) -\Pr_{\cD^{i-1}}[h_1(\bx)\neq
                                             h_2(\bx)]/4.
\end{align*}
It follows from the above and~\eqref{eq:upagree} that
\begin{eqnarray}
  \label{eq:dec}
  \er_{\cD^i}(h^\star_{\cD^i}) \leq 
(1/2)(\er_{\cD^i}(h_1) + \er_{\cD^i}(h_2) ) \leq \er_{\cD^{i-1}}(h^\star_{\cD^{i-1}}) -\Pr_{\cD^{i-1}}[h_1(\bx)\neq
h_2(\bx)]/8.
\end{eqnarray}
This is at most
\begin{align*}
   &\er_{\cD^{i-1}}(h^\star_{\cD^{i-1}}) -\Pr_{\cD^{i-1}}[h_1(\bx)\neq
  h_2(\bx)]/8 \\
  &\leq
  \er_{\cD^{i-1}}(h^\star_{\cD^{i-1}})
  -
  (1/16)\er_{\cD^{i-1}}(h^\star_{\cD^{i-1}})/\ln(1/\er_{\cD^{i-1}}(h^\star_{\cD^{i-1}}))
            \\
            &\leq
  (1-1/(16 \ln(1/\er_{\cD^{i-1}}(h^\star_{\cD^{i-1}}))))
  \er_{\cD^{i-1}}(h^\star_{\cD^{i-1}}).
\end{align*}
We now have two cases. If we already have
$\er_{\cD^{i-1}}(h^\star_{\cD^{i-1}}) \leq \tau^2$,
then for $\tau \leq 1/c_\tau$ (and using $t = O(\ln(1/\tau)\ln
\ln(1/\tau))$), we conclude
$\er_{\cD^{i}}(h^\star_{\cD^{i}}) \leq \tau(1-1/(32
\ln(1/\tau)))^{i-1}$ as claimed. If instead
$\er_{\cD^{i-1}}(h^\star_{\cD^{i-1}}) >
\tau^2$, we have 
\[
  (1-1/(16 \ln(1/\er_{\cD^{i-1}}(h^\star_{\cD^{i-1}})))) \leq (1-1/(32 \ln(1/\tau))) 
\]
It finally follows from the induction hypothesis that
\[
 \er_{\cD^{i}}(h^\star_{\cD^{i}}) \leq \tau(1-1/(32
\ln(1/\tau)))^{i-1}
\]
From~\eqref{eq:dec}, it also follows that:
\begin{align*}
  &\Pr_{\cD}[\exists j \leq i : h^j_1(\bx) \neq h^j_2(\bx)] - \Pr_{\cD}[\exists
  j \leq i-1 : h^j_1(\bx) \neq h^j_2(\bx)] \\
  &=
   \Pr_\cD[h^i_1(\bx) \neq h^i_2(\bx) \mid \forall j < i : h^j_1(\bx) =
  h^j_2(\bx)]\Pr_\cD[\forall j < i : h^j_1(\bx) = h^j_2(\bx)]  \\
  &\leq
  \Pr_\cD[h^i_1(\bx) \neq h^i_2(\bx) \mid \forall j < i : h^j_1(\bx) =
  h^j_2(\bx)]  \\
  &=
  \Pr_{\cD^{i-1}}[h^i_1(\bx) \neq h^i_2(\bx) ]  \\
  &\leq
  8\left(\er_{\cD^{i-1}}(h^\star_{\cD^{i-1}}) -
  \er_{\cD^i}(h^\star_{\cD^i}) \right).
\end{align*}
From the induction hypothesis, we conclude
\begin{align*}
  &\Pr_{\cD}[\exists j \leq i : h^j_1(\bx) \neq h^j_2(\bx)] \\
  &=
   \Pr_{\cD}[\exists j \leq i-1 : h^j_1(\bx) \neq h^j_2(\bx)] +
  (\Pr_{\cD}[\exists j \leq i : h^j_1(\bx) \neq
  h^j_2(\bx)]-\Pr_{\cD}[\exists j \leq i-1 : h^j_1(\bx) \neq h^j_2(\bx)])
                                             \\
                                             &\leq
  8\left(\er_{\cD}(h^\star_\cD) - \er_{\cD^{i-1}}(h^\star_{\cD^{i-1}})
  \right) + 8\left(\er_{\cD^{i-1}}(h^\star_{\cD^{i-1}}) -
  \er_{\cD^i}(h^\star_{\cD^i}) \right)  \\
  &=
  8\left(\er_{\cD}(h^\star_\cD) - \er_{\cD^{i}}(h^\star_{\cD^{i}})
  \right).
\end{align*}
The claim follows.
\end{proof}

\subsection{Tighter ERM for near-identical hypotheses (proof of Lemma~\ref{lem:betteruni})}
\label{sec:betteruni}
In this section, we prove that Empirical Risk Minimization performs better than the general ERM Theorem (Theorem~\ref{thm:ERM}) when the input distribution satisfies that all hypotheses in a set $\cH$ rarely disagree. We have restated Lemma~\ref{lem:betteruni} here for convenience:

\begin{customlem}{\ref{lem:betteruni}}
  Let $\cD$ be a distribution over $\cX$ and $\cH \subset \cX \to \{-1,1\}$ a hypothesis set of VC-dimension $d$. Assume there is a hypothesis $h_0 \in \cH$ such that for all $h \in \cH$, we have $\Pr_\cD[h(\bx) \neq h_0(\bx)] \leq p$. Then for any $0 < \delta < 1$, it holds with probability $1-\delta$ over a set $\bS$ of $n$ i.i.d.\ samples from $\cD$ that
  \begin{align*}
    &\sup_{h \in \cH} \left|\er_{\bS}(h) - \er_\cD(h) \right| \\
    &
    =\left| \er_{\bS}(h_0) - \er_\cD(h_0)\right| + O\left(\sqrt{p (\ln(1/p) d + \ln(1/\delta))/n} + (d\ln(n/d) + \ln(1/\delta))/n \right).
  \end{align*}
\end{customlem}

\begin{proof}[Proof of Lemma~\ref{lem:betteruni}]
We assume that $\cD$ is a distribution over $\cX \times \{-1,1\}$ for which the label $y$ is uniquely determined from $x$, i.e. one of $\Pr_\cD[\by=1 \mid \bx=x]$ and $\Pr_\cD[\by=-1 \mid \bx=x]$ is $0$ for all $x \in \cX$. This can be assumed wlog.\ by replacing each $x \in \cX$ with two point $x_{-1}$ and $x_{1}$ and letting the probability density function $p$ of $\cD$ satisfy $p((x_{-1},-1)) = p(x) \Pr_\cD[\by=-1 \mid \bx=x], p((x_{-1},1))=0$ and $p((x_1,1)) = p(x) \Pr_\cD[\by=1 \mid \bx=x], p((x_1,-1))=0$. Finally, for each $h \in \cH$ we let $h(x_{-1}) = h(x_1) = h(x)$. This leaves the VC-dimension of $\cH$ and $\er_\cD(h)$ for any $h$ unchanged. Let $c : \cX \to \{-1,1\}$ denote the concept giving the label of each $x \in \cX$.

Now consider the hypothesis set $\cH_= \subseteq \cX \to \{-1,1\}$ containing for each $h \in \cH$ the hypothesis $h_=$ taking the value $1$ on $x \in \cX$ with $h(x) \neq h_0(x) \wedge h(x)=c(x)$ and the value $-1$ otherwise. Also consider the set $\cH_{\neq}$ containing for each $h \in \cH$ the hypothesis $h_{\neq}$ taking the value $1$ on $x$ with $h(x) \neq h_0(x) \wedge h(x) \neq c(x)$, and $-1$ otherwise. The VC-dimension of $\cH_{=}$ and $\cH_{\neq}$ are both at most $O(d)$. Furthermore, every $h$ in $\cH_=$ and $\cH_{\neq}$ satisfy $\Pr_\cD[h(\bx)=1] \leq p$. Consider now the distribution $\cD'$ obtained by sampling $(\bx,\by) \sim \cD$ and replacing $\by$ by $-1$. Then $\er_{\cD'}(h)=\Pr_\cD[h(\bx)=1]$ for any $h \in \cH_=$ and $h \in \cH_{\neq}$. The ERM Theorem (Theorem~\ref{thm:ERM}) on $\cH_{=}$ and $\cH_{\neq}$ with distribution $\cD'$ implies that with probability $1-\delta$ over $\bS \sim \cD^n$, it holds that
  \[
    \sup_{h \in \cH_{=} \cup \cH_{\neq}} \left|\Pr_\bS[h(\bx)=1] - \Pr_\cD[h(\bx)=1] \right| \leq c_0 \cdot \left(\sqrt{p (d\ln(1/p) + \ln(1/\delta))/n} + (d\ln(n/d) + \ln(1/\delta))/n \right)
  \]
  for a constant $c_0>0$. Letting $\bS(\cdot \mid h_0 \neq h)$ denote the uniform distribution over samples $(x,y) \in \bS$ with $h_0(x)\neq h(x)$, and $\cD(\cdot \mid h_0 \neq h), \bS(\cdot \mid h_0=h), \cD(\cdot \mid h_0=h)$ defined symmetrically, we thus have for any $h \in \cH$,
  \begin{align*}
    &\left| \er_\bS(h) - \er_\cD(h) \right|  \\
    &=\left|\Pr_\bS[h_0 \neq h]\er_{\bS(\cdot \mid h_0 \neq h)}(h) + \Pr_\bS[h_0 = h]\er_{\bS(\cdot \mid h_0 = h)}(h) - \Pr_\cD[h_0 \neq h]\er_{\cD(\cdot \mid h_0 \neq h)}(h) - \Pr_\cD[h_0 = h]\er_{\cD(\cdot \mid h_0 = h)}(h) \right|  \\
    &=\left|\Pr_\bS[h_{\neq}=1]+ \Pr_\bS[h_0 = h]\er_{\bS(\cdot \mid h_0 = h)}(h) - \Pr_\cD[h_{\neq}=1]- \Pr_\cD[h_0 = h]\er_{\cD(\cdot \mid h_0 = h)}(h) \right| \\
    &=\left|\Pr_\bS[h_{\neq}=1]+ (\er_\bS(h_0) - \Pr_\bS[h_0 \neq h]\er_{S(\cdot \mid h_0 \neq h)}(h_0)) - \Pr_\cD[h_{\neq}=1]- (\er_\cD(h_0) - \Pr_\cD[h_0 \neq h]\er_{\cD(\cdot \mid h_0 \neq h)}(h) )\right|  \\
    &=\left|\Pr_\bS[h_{\neq}=1]+ (\er_\bS(h_0) - \Pr_\bS[h_{=}=1]) - \Pr_\cD[h_{\neq}=1]- (\er_\cD(h_0) - \Pr_\cD[h_{=} =1])\right| \\
    &\leq \left| \er_\bS(h_0) - \er_\cD(h_0)\right| + \left| \Pr_\bS[h_{\neq}=1] - \Pr_\cD[h_{\neq}=1]\right| +  \left| \Pr_\bS[h_{=}=1] - \Pr_\cD[h_{=}=1]\right|  \\
    &\leq \left| \er_\bS(h_0) - \er_\cD(h_0)\right| + 2c_0 \cdot \left(\sqrt{p (\ln(1/p) d + \ln(1/\delta))/n} + (d\ln(n/d) + \ln(1/\delta))/n \right).
  \end{align*}
  The claim follows.
\end{proof}

\section{Lower Bound for Proper Agnostic Learning}
In this section, we prove a lower bound for proper agnostic PAC learning, stated formally in Theorem~\ref{thm:mainlower}. So let $C>0$ be a sufficiently large constant, let $d$ be a target VC-dimension, $n$ a number of samples and let $\tau$ satisfy $Cd \ln(n/d)/n \leq \tau \leq 1/C$. Note that the assumption on $\tau$ also implies $n \geq C^2 d \ln(n/d)$.

We define the input domain $\cX$ to be the discrete domain $\cX = \{x_1,\dots,x_u\}$ for a $u \geq d$ to be determined. The hypothesis set $\cH$ contains all hypotheses $h$ that predict $-1$ on precisely $d$ of the points in $\cX$. The target concept $c : \cX \to \{-1,1\}$ to learn has $c(x)=1$ for all $x \in \cX$. 

Consider now an arbitrary \emph{proper} learning algorithm $\cA$ for $\cH$ and $c$. Recall that a proper learning algorithm always returns a hypothesis $h \in \cH$. Our goal is to show that there is a distribution $\cD$ over $\cX$ such that the hypothesis $h_\bS$ returned by $\cA$ on a sample $\bS \sim \cD^n$ often satisfies
\[
\er_\cD(h_\bS) = \tau + \Omega(\sqrt{\tau \ln(1/\tau)d/n})
\]
while there is a hypothesis $h \in \cH$ with $\er_\cD(h) = \tau$.

For proving this, we consider multiple distributions over $\cX$, one for each $h \in \cH$. For a hypothesis $h \in \cH$, the distribution $\cD_\cH$ returns a uniform point among those $x_i$ with $h(x_i)=-1$ with probability $(1-\alpha) d/u$ and it returns a uniform point among the remaining with probability $1-(1-\alpha) d /u$. Here $0 < \alpha < 1$ is a parameter to be determined. Observe that since $\alpha > 0$, we have that $h$ achieves the smallest error under $\cD_h$ among all $h \in \cH$. Furthermore
\begin{eqnarray}
\label{eq:opt}
    \er_{\cD_h}(h) = (1-\alpha) d/ u.
\end{eqnarray}
We will later fix $\alpha$ and $u$ such that $\tau = (1-\alpha) d/ u$, i.e. $\er_{\cD_h}(h)=\tau$.

To prove a lower bound on the error of $\cA$, we now draw a random $\randH \in \cH$ (unknown to $\cA$) and measure the performance of $\cA$ under the distribution $\cD_{\randH}$. For this, we consider the following \emph{failure} event of $\cA$. We say that $\cA$ fails on a sample $S$ from $\cD_{h}$ if it returns a hypothesis $h_S \in \cH$ such that $|\{i \in [u] : h_S(x_i)=h(x_i)=-1\}| \leq d/2$. We first observe
\begin{observation}
\label{obs:fail}
    If $\cA$ fails on a sample $S$ from $\cD^n_h$ for an $h \in \cH$, then 
    \[
    \er_{\cD_h}(h_S) \geq \tau + \alpha d/(2u).
    \]
\end{observation}
\begin{proof}
Since the target concept $c$ is the all-1 concept, we have that $\er_{\cD_h}(h_S) = \Pr_{x \sim \cD}[h_S(x)=-1]$. Every $i$ for which $h(x_i)=-1$ has $\Pr_{x \sim \cD}[x=x_i]=(1-\alpha)/u$ and every $i$ with $h(x_i)=1$ has $\Pr_{x \sim \cD}[x=x_i]=(1-(1-\alpha)d/u)/(u-d/u) > 1/u$. It follows that $\er_{\cD_h}(h_S) \geq (d/2)/u + (d/2)(1-\alpha)/u = d/u - (d/2)\alpha/u = \tau + \alpha d/(2u)$.
\end{proof}

The second part of the proof shows that $\cA$ fails with constant probability over a random choice of $\randH$ and sample $\bS \sim \cD_{\randH}^n$.

\begin{lemma}
\label{lem:oftenfail}
There is a universal constant $C'>2$, such that for any proper learning algorithm $\cA$, if $\alpha \leq \min\{\sqrt{u \ln(u/d)/(nC')}, 1/C'\}$ and $u$ satisfies $d C' \leq u \leq n/C'$, then it holds with probability at least $1/16$ over a random $\randH$ from $\cH$ and a random sample $\bS \sim \cD_{\randH}^n$ that $\cA$ fails on $\bS$.
\end{lemma}

Before giving the proof of Lemma~\ref{lem:oftenfail}, let us derive our lower bound. Recall that $\tau$ satisfies $Cd\ln(n/d)/n \leq \tau \leq 1/C$ for a sufficiently large constant $C>0$. Using~\eqref{eq:opt} and $\tau = \er_{\cD_\randH}(\randH)$, we let
\[
u = (1-\alpha)d/\tau.
\]
Also, fix $\alpha = \min\{\sqrt{u \ln(u/d)/(nC')},1/C'\}$ where $C'$ is the constant from Lemma~\ref{lem:oftenfail}. For $C$ large enough, we have $u=(1-\alpha)d/\tau \geq C (1-\alpha)d \geq C(1-1/C')d \geq dC'$. Similarly, we have $u=(1-\alpha)d/\tau \leq (1-\alpha)n/C \leq n/C \leq n/C'$. Thus $u$ and $\alpha$ satisfy the constraints in Lemma~\ref{lem:oftenfail}. It follows that for any proper learning algorithm $\cA$ and a random $\randH$ from $\cH$, $\cA$ fails with probability at least $1/16$ on a sample $\bS \sim \cD_{\randH}^n$. From Observation~\ref{obs:fail}, we get that in this case, $\er_{\cD_{\randH}}(h_{\bS}) \geq \tau + \alpha d/(2u)$. By our choice of parameters, we have
\[
\alpha d/(2u) = \min\{\sqrt{\ln(u/d)d^2/(4C'nu)}, d/(2C'u)\}
\]
Since $u=(1-\alpha)d/\tau \in [(1-1/C')d/\tau, d/\tau] \subseteq [d/(2\tau),d/\tau]$, this is at least
\[
\min\{\sqrt{\tau \ln(1/(2\tau))d /(4C' n)}, \tau/(2C')\}.
\]
For $\tau \geq Cd \ln(n/d)/n$ and $C$ large enough, the former term is the minimum and we conclude:
\[
\er_{\cD_\randH}(h_{\bS}) = \tau + \Omega\left(\sqrt{\frac{\tau d\ln(1/\tau) }{n}} \right).
\]
This concludes the proof of Theorem~\ref{thm:mainlower}. What remains is thus to establish Lemma~\ref{lem:oftenfail}:

\begin{proof}[Proof of Lemma~\ref{lem:oftenfail}]
It is not hard to see (naive Bayes) that the optimal strategy for any $\cA$ minimizing the probability of failure (over $\randH$ and $\bS$) when given a sample $\bS$, is to output the hypothesis $h_\bS$ returning $-1$ on the $d$ points among $x_1,\dots,x_u$ from which fewest samples were seen. We thus upper bound the probability that this set contains more than $d/2$ samples $x_i$ with $\randH(x_i)=-1$.

For this, fix an arbitrary outcome $h$ of $\randH$, let $U_h \subset \{x_1,\dots,x_n\}$ denote the set of $x_i$ with $h(x_i)=-1$. We have $|U_h|=d$. For any $x_i \in U_h$, let $\bn_i$ give the number of copies of $x_i$ in the sample $\bS \sim \cD_h^n$. We have $\E[\bn_i]=(1-\alpha)n/u$. Since $\bn_i$ is a sum of independent indicator random variables, its variance is at most its expectation. It follows from Chebyshev's inequality that $\Pr[|\bn_i - \E[\bn_i]|>8 \sqrt{\E[\bn_i]}] \leq 1/64$. Markov's inequality implies that with probability at least $15/16$, there are no more than $d/4$ indices $i \in U_h$ for which $\bn_i < (1-\alpha)n/u - 8 \sqrt{(1-\alpha)n/u}$. 

We next show that there is a good chance that at least $d/2$ of the points $x_i$ with $x_i \notin U_h$ have $\bn_i < (1-\alpha)n/u - 8 \sqrt{(1-\alpha)n/u}$. For this, we require $u \geq 2d$. Then for any $x_i \notin U_h$, we have $\Pr_{x \sim \cD_h}[x=x_i] = q/u$ for $q \in [1, 1+\alpha]$. We now invoke the following anti-concentration result:
\begin{lemma}
\label{lem:smallbins}
Consider throwing $n$ balls independently into $u$ bins, such that the $i$'th bin is hit with probability $p_i$. Let $S$ be a subset of $m$ bins such that each bin in $S$ has $p_i = p$ for some $12/n \leq p \leq 1/2$. Then for every integer $k \leq m/C$ for a sufficiently large constant $C>0$, it holds with probability at least $1/8$ that there are at least $k$ bins in $S$ containing less than $\max\{pn - \sqrt{pn \ln(m/k)}/6,pn/2\}$ balls each.
\end{lemma}
We prove the lemma in Appendix~\ref{sec:reverse}.

By Lemma~\ref{lem:smallbins} (setting $p=q/u$, $m=u-d \geq u/2$, $k=d/2$), with probability at least $1/8$ provided $Cd \leq u \leq n/C$ for a large enough constant $C>0$, there are at least $d/2$ points $x_i$ with $h(x_i)=1$ such that we see no more than
\[
qn/u - \min\{\sqrt{(qn/u)\ln(u/d)}/6, qn/(2u)\} \leq (1+\alpha)n/u - \min\{\sqrt{(n/u)\ln(u/d)}/6, n/(2u)\}
\]
copies of $x_i$. 

We aim to choose $\alpha$ such that $\min\{\sqrt{(n/u)\ln(u/d)}/6, n/(2u)\} \geq 2\alpha n/u + 8\sqrt{n/u}$, since then the number of copies we see of these $x_i$ is no more than $(1-\alpha)n/u-8\sqrt{n/u}$. 

For this, we first constrain $u$ to satisfy $Cd \leq u \leq n/C$ for a big enough constant $C>0$ so that 
\[
8 \sqrt{n/u} \leq \max\{\sqrt{(n/u)\ln(u/d)}/12, n/(4u)\}.
\]
The constraint on $\alpha$ is now satisfied when 
\[
2\alpha n/u \leq \min\{\sqrt{(n/u)\ln(u/d)}/12, n/(4u) \}
\]
which is
\[
  \alpha \leq \min\{\sqrt{u \ln(u/d)/(24^2 n)}, 1/8\}.
\]
We conclude that for such $\alpha$ and $u$, with probability at least $1/8-1/16=1/16$ over $\bS$ and $\randH$, we have that $\cA$ fails on $\bS$.
\end{proof}

\section{Conclusion and Open Problems}
In this work, we established that ERM, and all other proper learning algorithms, are sub-optimal for agnostic PAC learning when treating $\tau = \er_\cD(h^\star_\cD)$ as a parameter. We then complemented the lower bound with a new improper learning algorithm that achieves an optimal sample complexity except for very small values of $\tau$. However, a number of intriguing questions remain. First, can we develop an algorithm that is optimal for the full range of $\tau$? In particular, our new algorithm implies that it suffices to consider the near-realizable case of $\tau = O(\ln^9(n/d)d/n)$. Secondly, we know that variants of majority voting (bagging, etc.) are optimal for realizable PAC learning. The analysis tools used when proving their optimality breaks down for the agnostic setting. Can we somehow analyse them in a different way and prove that e.g.\ bagging is optimal both in the agnostic and realizable setting? If not, can we prove a lower bound for concrete algorithms, such as bagging or Hanneke's majority voter, proving that they are sub-optimal in the agnostic case? Thirdly, it could be the case that there is a higher lower bound for all learning algorithms when $\tau \approx d/n$. Can we prove this? Another interesting question is whether we can design an optimal agnostic learning algorithm that automatically adapts to $\delta$? In more detail, our new algorithm requires knowledge of the failure probability $\delta$ and thus works only for a fixed user defined $\delta$. ERM on the other hand automatically works for all values of $\delta$ simultaneously. Next, our algorithm is not necessarily efficient in terms of running time. In particular, even if ERM over $\cH$ is efficient, it is unclear how to determine if there are two hypotheses $h_1,h_2$ that are both near-optimal and yet disagree in the classification of many samples. Can we design an efficient learning algorithm with sample complexity similar to our new algorithm? Finally, the authors find the idea of recursively training near-optimal, but highly disagreeing classifiers, to be promising. Are there other applications of this idea in learning theory?

\section*{Acknowledgment}
Kasper Green Larsen is co-funded by the European Union (ERC, TUCLA, 101125203) and Independent Research Fund Denmark (DFF) Sapere Aude Research Leader Grant No. 9064-00068B. Views and opinions expressed are however those of the author(s) only and do not necessarily reflect those of the European Union or the European Research Council. Neither the European Union nor the granting authority can be held responsible for them.

\bibliography{refs}

\begin{thebibliography}{10}

\bibitem{Aden-AliFOCS23}
I.~Aden{-}Ali, Y.~Cherapanamjeri, A.~Shetty, and N.~Zhivotovskiy.
\newblock Optimal {PAC} bounds without uniform convergence.
\newblock In {\em 2023 IEEE 64th Annual Symposium on Foundations of Computer
  Science (FOCS)}, pages 1203--1223. IEEE Computer Society, 2023.

\bibitem{maj3}
I.~Aden-Ali, M.~M. H{\o}gsgaard, K.~G. Larsen, and N.~Zhivotovskiy.
\newblock Majority-of-three: The simplest optimal learner?
\newblock {\em CoRR}, abs/2403.08831, 2024.

\bibitem{bookvc}
M.~Anthony and P.~L. Bartlett.
\newblock {\em Neural Network Learning: Theoretical Foundations}.
\newblock Cambridge University Press, USA, 1st edition, 2009.

\bibitem{audibert2009fast}
J.-Y. Audibert.
\newblock Fast learning rates in statistical inference through aggregation.
\newblock {\em The Annals of Statistics}, 37(4):1591--1646, 2009.

\bibitem{auer2007new}
P.~Auer and R.~Ortner.
\newblock A new {PAC} bound for intersection-closed concept classes.
\newblock {\em Machine Learning}, 66(2):151--163, 2007.

\bibitem{blumer1989learnability}
A.~Blumer, A.~Ehrenfeucht, D.~Haussler, and M.~K. Warmuth.
\newblock Learnability and the {V}apnik-{C}hervonenkis dimension.
\newblock {\em Journal of the ACM (JACM)}, 36(4):929--965, 1989.

\bibitem{boucheron2005theory}
S.~Boucheron, O.~Bousquet, and G.~Lugosi.
\newblock Theory of classification: A survey of some recent advances.
\newblock {\em ESAIM: Probability and Statistics}, 9:323--375, 2005.

\bibitem{bousquet2020proper}
O.~Bousquet, S.~Hanneke, S.~Moran, and N.~Zhivotovskiy.
\newblock Proper learning, {H}elly number, and an optimal {SVM} bound.
\newblock In {\em Conference on Learning Theory}, pages 582--609. PMLR, 2020.

\bibitem{bousquet2021fast}
O.~Bousquet and N.~Zhivotovskiy.
\newblock Fast classification rates without standard margin assumptions.
\newblock {\em Information and Inference: A Journal of the IMA},
  10(4):1389--1421, 2021.

\bibitem{Breiman2004BaggingP}
L.~Breiman.
\newblock Bagging predictors.
\newblock {\em Mach. Learn.}, 24(2):123–140, aug 1996.

\bibitem{DevroyeGyorfiLugosi1996}
L.~Devroye, L.~Gy{\"o}rfi, and G.~Lugosi.
\newblock {\em A Probabilistic Theory of Pattern Recognition}.
\newblock Springer, New York, 1996.

\bibitem{ehrenfeucht1989general}
A.~Ehrenfeucht, D.~Haussler, M.~Kearns, and L.~Valiant.
\newblock A general lower bound on the number of examples needed for learning.
\newblock {\em Information and Computation}, 82(3):247--261, 1989.

\bibitem{marginsGradient}
A.~Gr{\o}nlund, L.~Kamma, and K.~G. Larsen.
\newblock Margins are insufficient for explaining gradient boosting.
\newblock In {\em Advances in Neural Information Processing Systems 33 (NeurIPS
  2020)}, 2020.

\bibitem{hanneke2016optimal}
S.~Hanneke.
\newblock The optimal sample complexity of pac learning.
\newblock {\em The Journal of Machine Learning Research}, 17(1):1319--1333,
  2016.

\bibitem{hanneke2016refined}
S.~Hanneke.
\newblock Refined error bounds for several learning algorithms.
\newblock {\em The Journal of Machine Learning Research}, 17(1):4667--4721,
  2016.

\bibitem{HAUSSLER199278}
D.~Haussler.
\newblock Decision theoretic generalizations of the pac model for neural net
  and other learning applications.
\newblock {\em Information and Computation}, 100(1):78--150, 1992.

\bibitem{haussler1994predicting}
D.~Haussler, N.~Littlestone, and M.~K. Warmuth.
\newblock Predicting $\{$0, 1$\}$-functions on randomly drawn points.
\newblock {\em Information and Computation}, 115(2):248--292, 1994.

\bibitem{reversechernoff}
P.~N. Klein and N.~E. Young.
\newblock On the number of iterations for dantzig-wolfe optimization and
  packing-covering approximation algorithms.
\newblock {\em {SIAM} J. Comput.}, 44(4):1154--1172, 2015.

\bibitem{larsen2023baggingCOLT}
K.~G. Larsen.
\newblock Bagging is an optimal {PAC} learner.
\newblock In {\em The Thirty Sixth Annual Conference on Learning Theory, {COLT}
  2023}, volume 195 of {\em Proceedings of Machine Learning Research}, pages
  450--468. {PMLR}, 2023.

\bibitem{lls}
Y.~Li, P.~Long, and A.~Srinivasan.
\newblock Improved bounds on the sample complexity of learning.
\newblock {\em Journal of Computer and System Sciences}, 62:516 -- 527, 2001.

\bibitem{massart2006risk}
P.~Massart and E.~N\'{e}d\'{e}lec.
\newblock Risk bounds for statistical learning.
\newblock {\em The Annals of Statistics}, 34, 10 2006.

\bibitem{puchkin2021exponential}
N.~Puchkin and N.~Zhivotovskiy.
\newblock Exponential savings in agnostic active learning through abstention.
\newblock In {\em Conference on Learning Theory}, pages 3806--3832. PMLR, 2021.

\bibitem{raginsky2011lower}
M.~Raginsky and A.~Rakhlin.
\newblock Lower bounds for passive and active learning.
\newblock {\em Advances in Neural Information Processing Systems}, 24, 2011.

\bibitem{simon2015almost}
H.~U. Simon.
\newblock An almost optimal {PAC} algorithm.
\newblock In {\em Conference on Learning Theory}, pages 1552--1563. PMLR, 2015.

\bibitem{valiant1984theory}
L.~G. Valiant.
\newblock A theory of the learnable.
\newblock {\em Communications of the ACM}, 27(11):1134--1142, 1984.

\bibitem{vapnik:estimation}
V.~Vapnik.
\newblock {\em Estimation of Dependences Based on Empirical Data: Springer
  Series in Statistics (Springer Series in Statistics)}.
\newblock Springer-Verlag, Berlin, Heidelberg, 1982.

\bibitem{vapnik1964class}
V.~Vapnik and A.~Chervonenkis.
\newblock A class of algorithms for pattern recognition learning.
\newblock {\em Avtomatika i Telemekhanika}, 25(6):937--945, 1964.

\bibitem{vapnik74theory}
V.~Vapnik and A.~Chervonenkis.
\newblock {\em Theory of Pattern Recognition}.
\newblock Nauka, Moscow, 1974.

\bibitem{vapnik71uniform}
V.~N. Vapnik and A.~Y. Chervonenkis.
\newblock On the uniform convergence of relative frequencies of events to their
  probabilities.
\newblock {\em Theory of Probability and its Applications}, 16(2):264--280,
  1971.

\bibitem{zhivotovskiy2018localization}
N.~Zhivotovskiy and S.~Hanneke.
\newblock Localization of vc classes: Beyond local rademacher complexities.
\newblock {\em Theoretical Computer Science}, 742:27--49, 2018.

\end{thebibliography}
\bibliographystyle{abbrv} 

\appendix

\section{Appendix}
\label{sec:omit}

\subsection{Failures are rare (proof of Lemma~\ref{lem:unlikely})}
\label{sec:unlikely}
In this section, we show that the failure events $E_{i,0}$ and $E_{i,1}$ are unlikely. Formally, we prove:
\begin{customlem}{\ref{lem:unlikely}}
    For all $i$, we have $\Pr[E_{i,0}] \leq \delta/(4 t)$ and $\Pr[E_{i,1}] \leq \delta/(4t)$.
\end{customlem}

\begin{proof}[Proof of Lemma~\ref{lem:unlikely}]
We consider each type of event in turn. Since both of the events only occur when none of the events $E_{j,0}, E_{j,1}$ occurred for any $j<i$, we bound the probability of $E_{i,0}$ and $E_{i,1}$ under this assumption. So fix an outcome $B^1,\dots,B^{i-1}$ of $\bB^1,\dots,\bB^{i-1}$ such that none of the events occurred for $j<i$. This also fixes an outcome $h_1^j,h_2^j$ of $\bh^j_1,\bh^j_2$ for $j < i$ and $\cD^i$ of $\bD^i$. Note that $\bB^i$ is independent of the events $E_{j,0},E_{j,1}$ for $j<i$ and thus $\bB^i$ still consists of $n/t$ i.i.d.\ samples from $\cD$.

Observe that by Lemma~\ref{lem:progress}, $\Pr_{\cD}[\exists j < i : h^j_1(\bx) \neq h^j_2(\bx)] \leq 8 \tau$. Since $\tau \leq 1/c_\tau$ for large enough $c_\tau>0$, this further implies that $\Pr_{\cD}[\forall j < i : h^j_1(\bx) = h^j_2(\bx)] \geq 1-8\tau \geq 1/2$. We further have $|\bB^{i}| = n/t > c_n \ln(t/\delta)$ by assumptions $n \geq c_n \ln^{3.5}(n/d)(d + \ln(1/\delta))$, $t = O(\ln(1/\tilde{\tau})\ln\ln(1/\tilde{\tau})) = O(\ln(n/d)\ln \ln(n/d))$, $\tilde{\tau} \in [\tau/2,2\tau]$ and $\tau > d/n$. It follows that $|\bT^i| \geq n/(2t)$ except with probability $\delta/(8t)$. We thus bound the probabilities under the assumption that $\bT^i$ consists of $m \ge n/(2t)$ i.i.d.\ samples from $\cD^i$.

\paragraph{Event $E_{i,0}$.}
By the ERM Theorem (Theorem~\ref{thm:ERM}) on $\bT^i$, we have with probability $1-\delta/(8 t)$ that for all $h \in \cH$
\begin{eqnarray}
\label{eq:erm}
|\er_{\bT^i}(h) - \er_{\cD^i}(h)| =O\left(\sqrt{\frac{t\er_{\cD^i}(h)(d \ln(n/d) + \ln(t/\delta))}{n}} + \frac{t(d \ln(n/d) + \ln(t/\delta))}{n}\right).
\end{eqnarray}
The $\ln(t/\delta)=\ln(t) + \ln(1/\delta)$ term may be replaced by $\ln(1/\delta)$ as $\ln(t)$ is dominated by the $d \ln(n/d)$ term. Thus for the constant $c_\alpha$ in the definition of $\alpha$ large enough, this implies 
\[
|\er_{\bT^i}(h) - \er_{\cD^i}(h)| \leq (1/32)\alpha(n/t,d,\delta,\min\{\er_{\cD^i}(h), \er_{\bT^i}(h)\}).
\]
To see that we may insert $\min\{\er_{\cD^i}(h), \er_{\bT^i}(h)\}$ instead of $\er_{\cD^i}(h)$, we consider two cases. First, if $\er_{\cD^i}(h)>\sqrt{c_\alpha} t(d \ln(n/d) + \ln(t/\delta))$ for a sufficiently large  constant $c_\alpha$ in the definition of $\alpha$, we have from~\eqref{eq:erm} that $\er_{\bT^i}(h) \geq (1/2)\er_{\cD^i}(h)$ and thus $|\er_{\bT^i}(h) - \er_{\cD^i}(h)| \leq (1/32)\alpha(n/t,d,\delta,\min\{\er_{\cD^i}(h), \er_{\bT^i}(h)\})$. If on the other hand $\er_{\cD^i}(h) \leq \sqrt{c_\alpha} t(d \ln(n/d) + \ln(t/\delta))$ then we still have $(1/32)\alpha(n/t,d,\delta,0) \geq |\er_{\bT^i}(h)-\er_{\cD^i}(h)|$ by~\eqref{eq:erm} and large enough constant $c_\alpha$.

\paragraph{Event $E_{i,1}$.}
Consider the set of hypotheses $\cH' = \cH \oplus \cH$ consisting of all hypotheses that may be written as $g_{h_1,h_2}(x) = h_1(x) \cdot h_2(x)$ where $\cdot$ denotes multiplication and $h_1,h_2 \in \cH$. Then $\Pr_{\cD}[g_{h_1,h_2}(\bx)=1] = \Pr_{\cD}[h_1(\bx) \neq h_2(\bx)]$ for any distribution $\cD$. Furthermore, the VC-dimension of $\cH'$ is $O(d)$. 

It follows by the ERM Theorem (Theorem~\ref{thm:ERM}) and the constant $c_\alpha$ in the definition of $\alpha$ large enough, that with probability at least $1-\delta/(8t)$, any pair $h_1,h_2$ satisfy $|\Pr_{\cD^i}[h_1(\bx)\neq h_2(\bx)] - \Pr_{\bT^i}[h_1(\bx)\neq h_2(\bx)]| \leq (1/32)\alpha(n/t,d,\delta,\min\{\Pr_{\cD^i}[h_1(\bx)\neq h_2(\bx)], \Pr_{\bT^i}[h_1(\bx) \neq h_2(\bx)]\})$. Here we may insert the $\min$ by the same arguments as above.
\end{proof}

\subsection{Properties when no failures (proof of Observation~\ref{obs:props})}
In this section, we show that when the events $E_{j,0}$ and $E_{j,1}$ do not occur, the hypotheses in $\bH^i$ and $\cH$ behave nicely. Concretely, we prove Observation~\ref{obs:props}, which we have restated here:

\begin{customobs}{\ref{obs:props}}
Assume none of the events $E_{j,0}$ and $E_{j,1}$ occur for $j \leq i$ and that Algorithm~\ref{alg:agnostic} does not terminate before iteration $i$. Then if $\gamma_i = \er_{\bT^i}(h^\star_{\bT^i}) \leq Z_t$ it holds that $\er_{\bD^i}(h^\star_{\bD^i}) \leq 2Z_t$. If $\gamma_i > Z_t$, then each of the following hold:
\begin{itemize}
\item
$
\gamma_i \leq \er_{\bD^{i}}(h^\star_{\bD^{i}})+\er_{\bD^{i}}(h^\star_{\bD^{i}})/\ln(1/\er_{\bD^{i}}(h^\star_{\bD^{i}})) \leq 2\er_{\bD^{i}}(h^\star_{\bD^{i}}).
$

\item Every hypothesis $h$ in $\bH^{i}$ satisfies $\er_{\bD_{i}}(h) \leq  \er_{\bD^i}(h^\star_{\bD^i}) + (1/8)\er_{\bD^i}(h^\star_{\bD^i})/\ln(1/\er_{\bD^i}(h^\star_{\bD^i}))$.

\item Every hypothesis $h \in \cH$ with $\er_{\bD_i}(h) \leq \er_{\bD^i}(h^\star_{\bD^i}) + (1/8)\alpha(n/t,d,\delta,\er_{\bD^i}(h^\star_{\bD^i}))$ is in $\bH^i$.

\item Every pair of hypotheses $h_1,h_2$ with $\Pr_{\bT^i}[h_1(\bx) \neq h_2(\bx)] \geq \gamma_i/\ln(1/\gamma_i)$ satisfy $\Pr_{\bD^i}[h_1(\bx) \neq h_2(\bx)] \geq (1/2)\er_{\bD^i}(h^\star_{\bD^i})/\ln(1/\er_{\bD^i}(h^\star_{\bD^i}))$.

\item Every pair of hypotheses $h_1,h_2$ with $\Pr_{\bT^i}[h_1(\bx) \neq h_2(\bx)] < \gamma_i/\ln(1/\gamma_i)$ satisfy $\Pr_{\bD^i}[h_1(\bx) \neq h_2(\bx)] \leq 4\er_{\bD^i}(h^\star_{\bD^i})/\ln(1/\er_{\bD^i}(h^\star_{\bD^i}))$.
\end{itemize}
\end{customobs}

Before proving Observation~\ref{obs:props}, we state and prove an auxiliary result regarding $\alpha$ and $Z_t$:

\label{sec:props}
\begin{observation}
\label{obs:Z}
For $x \geq Z_t/2$, we have
\begin{eqnarray*}
 \alpha(n/t,d, \delta, x) \leq \frac{2 c_\alpha}{\sqrt{c_Z}} \cdot \frac{x}{\ln(1/x)}.
 \end{eqnarray*}
\end{observation}

\begin{proof}[Proof of Observation~\ref{obs:Z}]
Notice that for any $x \geq Z_t/2$,
we have $\ln(1/x) \leq
\ln(n/d)$ and thus for $x \geq Z_t/2$:
\begin{eqnarray*}
  \alpha(n/t,d, \delta, x) &\leq& c_\alpha \cdot \left(\sqrt{\frac{tx(d
                               \ln(1/x) + \ln(1/\delta))}{n} }+
                               \frac{t(d \ln(n/d) + \ln(1/\delta))}{n }
                               \right) \\
  &=& c_\alpha \cdot \left(\sqrt{\frac{tx \ln^2(1/x)(d
                               \ln(1/x) + \ln(1/\delta))}{n \ln^2(1/x)} }+
                               \frac{t \ln(1/x)(d\ln(n/d)  +
      \ln(1/\delta))}{n \ln(1/x)}
      \right) \\
  &\leq&
  c_\alpha \cdot \left(\sqrt{\frac{tx \ln^2(n/d)(d
                               \ln(n/d) + \ln(1/\delta))}{n \ln^2(1/x)} }+
                               \frac{t \ln(n/d)(d \ln(n/d) +
      \ln(1/\delta))}{n \ln(1/x)}
         \right)  \\
  &\leq&
  c_\alpha \cdot \left(\sqrt{\frac{x^2}{c_Z\ln^2(1/x)} }+
         \frac{ x}{c_Z\ln(1/x)}\right) \\
                           &\leq&
                                  \frac{2 c_\alpha}{\sqrt{c_Z}} \cdot \frac{x}{\ln(1/x)}.
\end{eqnarray*}
\end{proof}

Let us also restate the failure events $E_{i,0}$ and $E_{i,1}$ here for convenience:
\begin{enumerate}
\item Let $E_{i,0}$ be the event that Algorithm~\ref{alg:agnostic} reaches iteration $i$, none of the events $E_{j,0},E_{j,1}$ occurred for $i<j$ and there is a hypothesis $h \in \cH$ with $|\er_{\bD^i}(h) - \er_{\bT^i}(h)| > (1/32)\alpha(n/t,d,\delta,\min\{\er_{\bD^i}(h),\er_{\bT^i}(h)\})$.

\item Let $E_{i,1}$ be the event that Algorithm~\ref{alg:agnostic} reaches iteration $i$, none of the events $E_{j,0},E_{j,1}$ occurred for $i<j$ and there is a pair of hypotheses $h_1,h_2$ with $|\Pr_{\bD^i}[h_1(\bx)\neq h_2(\bx)] - \Pr_{\bT^i}[h_1(\bx)\neq h_2(\bx)]| > (1/32)\alpha(n/t,d,\delta,\min\{\Pr_{\bD^i}[h_1(\bx)\neq h_2(\bx)], \Pr_{\bT^i}[h_1(\bx) \neq h_2(\bx)]\})$.
\end{enumerate}
We are ready to prove Observation~\ref{obs:props}.

\begin{proof}[Proof of Observation~\ref{obs:props}]
Since we only claim something when the events $E_{j,0}$ and $E_{j,1}$ did not occur for $j \leq i$ and that Algorithm~\ref{alg:agnostic} did not terminate before iteration $i$, we assume this. So fix such an outcome $B^1,\dots,B^i$ of $\bB^1,\dots,\bB^i$. This also fixes an outcome $h^j_1,h^j_2, T^j, \cD^j, \cH^j$ of $\bh^j_1,\bh^j_2,\bT^j,\bD^j, \bH^j$ for $j \leq i$.

Assume first that $\er_{\cD^i}(h^\star_{\cD^i}) > 2Z_t$. We wish to show $\gamma_i > Z_t$. To see this, note that by definition of $E_{i,0}$, we have $\er_{T^i}(h^\star_{T^i}) \geq 2Z_t - (1/32)\alpha(n/t,d,\delta,2Z_t)$. By Observation~\ref{obs:Z}, for large enough $c_Z$, this is at least $2Z_t -Z_t/2 > Z_t$.  This proves the part of Observation~\ref{obs:props} stating that if $\gamma_i = \er_{T^i}(h^\star_{T^i}) \leq Z_t$ then $\er_{\cD^i}(h^\star_{\cD^i}) \leq 2Z_t$.

For the remainder of the proof, assume $\gamma_i > Z_t$. We start by proving bounds on $\er_{\cD^{i}}(h^\star_{\cD^{i}})$.

Since the event $E_{i,0}$ did not occur, we have
\begin{itemize}
\item $\er_{T^{i}}(h^\star_{\cD^{i}}) \geq \gamma_i \geq Z_t$ and thus $\er_{\cD^{i}}(h^\star_{\cD^{i}}) \geq Z_t - \alpha(n/t,d,\delta,\er_{\cD^{i}}(h^\star_{\cD^{i}}))$. We claim this inequality implies $\er_{\cD^{i}}(h^\star_{\cD^{i}}) \geq Z_t/2$. To see this, assume for contradiction that $\er_{\cD^{i}}(h^\star_{\cD^{i}}) < Z_t/2$, then since $\alpha$ is increasing in its last argument, we have by Observation~\ref{obs:Z} that $\er_{\cD^{i}}(h^\star_{\cD^{i}}) + \alpha(n/t,d,\delta,\er_{\cD^{i}}(h^\star_{\cD^{i}})) \leq Z_t/2 + \alpha(n/t,d,\delta,Z_t/2)$. For $c_Z$ large enough, the right hand side is less than $Z_t$, which contradicts the inequality $\er_{\cD^{i}}(h^\star_{\cD^{i}}) \geq Z_t - \alpha(n/t,d,\delta,\er_{\cD^{i}}(h^\star_{\cD^{i}}))$. 
\end{itemize}
In summary, we have
\begin{eqnarray}
\label{eq:bigerr}
    \er_{\cD^{i}}(h^\star_{\cD^{i}}) \geq Z_t/2.
\end{eqnarray}

Next, we upper bound $\gamma_i$.
\begin{itemize}
\item Since $E_{i,0}$ did not occur, we have 
\[
\gamma_i = \er_{T^i}(h^\star_{T^i}) \leq \er_{T^{i}}(h^\star_{\cD^{i}}) \leq \er_{\cD^{i}}(h^\star_{\cD^{i}}) + (1/32)\alpha(n/t,d,\delta,\er_{\cD^{i}}(h^\star_{\cD^{i}})).
\]
By Observation~\ref{obs:Z} for large enough $c_Z$ and using~\eqref{eq:bigerr}, this implies
\[
\gamma_i \leq \er_{\cD^{i}}(h^\star_{\cD^{i}})+\er_{\cD^{i}}(h^\star_{\cD^{i}})/\ln(1/\er_{\cD^{i}}(h^\star_{\cD^{i}})) \leq 2\er_{\cD^{i}}(h^\star_{\cD^{i}}).
\]
This establishes the first bullet in Observation~\ref{obs:props}.

\item To establish the second bullet of Observation~\ref{obs:props}, note that by definition of $\cH^i$, every hypothesis $h$ in $\cH^{i}$ satisfies $\er_{T_{i}}(h) \leq \gamma_{i} + \alpha(n/t,d,\delta,\gamma_{i})$ and since $\gamma_i \leq 2 \er_{\cD^{i}}(h^\star_{\cD^{i}})$ (by the previous bullet) and $\alpha(n/t,d,\delta,2x) \leq 2 \alpha(n/t,d,\delta,x)$ we have $\er_{T_{i}}(h) \leq \gamma_{i} + 2\alpha(n/t,d,\delta,\er_{\cD^{i}}(h^\star_{\cD^{i}}))$. Also using that $\gamma_i \leq \er_{\cD^i}(h^\star_{\cD^i}) + (1/32)\alpha(n/t,d,\delta,\er_{\cD^i}(h^\star_{\cD^i}))$ gives $\er_{T_{i}}(h) \leq \er_{\cD^i}(h^\star_{\cD^i}) + 3\alpha(n/t,d,\delta,\er_{\cD^i}(h^\star_{\cD^i})) \leq 6\er_{\cD^i}(h^\star_{\cD^i})$. Finally, since $E_{i,0}$ did not occur, we have $\er_{\cD_{i}}(h) \leq \er_{T_{i}}(h) + (1/32)\alpha(n/t,d,\delta,\er_{T_{i}}(h)) \leq \er_{T_{i}}(h) + (1/32)\alpha(n/t,d,\delta,6\er_{\cD^i}(h^\star_{\cD^i})) \leq \er_{\cD^i}(h^\star_{\cD^i}) + (1/5)\alpha(n/t,d,\delta,\er_{\cD^i}(h^\star_{\cD^i}))$. By~\eqref{eq:bigerr}, it follows by Observation~\ref{obs:Z} that $\alpha(n/t,d,\delta,\er_{\cD^i}(h^\star_{\cD^i})) \leq \frac{2c_\alpha}{\sqrt{c_Z}} \er_{\cD^i}(h^\star_{\cD^i})/\ln(1/\er_{\cD^i}(h^\star_{\cD^i}))$. For $c_Z$ large enough, we thus have $\er_{\cD_{i}}(h) \leq \er_{\cD^i}(h^\star_{\cD^i}) + (1/8)\er_{\cD^i}(h^\star_{\cD^i})/\ln(1/\er_{\cD^i}(h^\star_{\cD^i}))$. This establishes the second bullet of Observation~\ref{obs:props}.

\item For the third bullet, let $h$ be a hypothesis with $\er_{\cD^i}(h) \leq \er_{\cD^i}(h^\star_{\cD^i}) + (1/8)\alpha(n/t,d,\delta,\er_{\cD^i}(h^\star_{\cD^i}))$. Since $E_{i,0}$ did not occur, we have $\er_{T^i}(h) \leq \er_{\cD^i}(h) + (1/32)\alpha(n/t,d,\delta,\er_{\cD^i}(h))$. By Observation~\ref{obs:Z} and~\eqref{eq:bigerr}, we have $(1/8)\alpha(n/t,d,\delta,\er_{\cD^i}(h^\star_{\cD^i})) \leq \er_{\cD^i}(h^\star_{\cD^i})$ for $c_Z$ large enough. Thus we have $\er_{\cD^i}(h) \leq 2 \er_{\cD^i}(h^\star_{\cD^i})$. This gives $\er_{T^i}(h) \leq \er_{\cD^i}(h) + (1/16)\alpha(n/t,d,\delta,\er_{\cD^i}(h^\star_{\cD^i})) \leq \er_{\cD^i}(h^\star_{\cD^i}) + (3/8)\alpha(n/t,d,\delta,\er_{\cD^i}(h^\star_{\cD^i}))$. Finally, since $E_{i,0}$ did no occur, we have $\gamma_i \geq \er_{\cD^i}(h^\star_{\cD^i}) - (1/32)\alpha(n/t,d,\delta,\er_{\cD^i}(h^\star_{\cD^i}))$, implying $\er_{T^i}(h) \leq \gamma_i + (1/2)\alpha(n/t,d,\delta,\er_{\cD^i}(h^\star_{\cD^i}))$. We also have $\gamma_i \geq (1/2)\er_{\cD^i}(h^\star_{\cD^i})$ by the first bullet. Hence we conclude  $\er_{T^i}(h) \leq \gamma_i + \alpha(n/t,d,\delta,\gamma_i)$, which puts $h$ in $\cH^i$.

\item For the fourth bullet, consider a pair of hypotheses $h_1,h_2$ with $\Pr_{T^i}[h_1(\bx) \neq h_2(\bx)] \geq \gamma_i/\ln(1/\gamma_i)$. 
Since $E_{i,1}$ did not occur, we have 
\[
\Pr_{\cD^i}[h_1(\bx) \neq h_2(\bx)] \geq \Pr_{T^i}[h_1(\bx) \neq h_2(\bx)] - (1/32)\alpha(n/t,d,\delta,\Pr_{T^i}[h_1(\bx) \neq h_2(\bx)]).
\]
Since $\gamma_i/\ln(1/\gamma_i) \geq Z_t/\ln(1/Z_t)$ and the expression $x-(1/32)\alpha(n/t,d,\delta,x)$ is increasing in $x$ for $x > Z_t/\ln(1/Z_t)$ and $c_Z$ large enough, we conclude 
\[
\Pr_{\cD^i}[h_1(\bx) \neq h_2(\bx)] \geq \gamma_i/\ln(1/\gamma_i) - \alpha(n/t,d,\delta,\gamma_i/\ln(1/\gamma_i)) \geq \gamma_i/\ln(1/\gamma_i) - \alpha(n/t,d,\delta,\gamma_i).
\]
Since $\gamma_i \geq Z_t$, it follows from Observation~\ref{obs:Z} that this is at least $(3/4)\gamma_i/\ln(1/\gamma_i)$ for $c_Z$ large enough. Finally, since $\gamma_i \geq \er_{\cD^i}(h^\star_{\cD^i}) - \alpha(n/t,d,\delta,\er_{\cD^i}(h^\star_{\cD^i}))$. We have $\er_{\cD^i}(h^\star_{\cD^i}) \geq Z_t/2$ so by Observation~\ref{obs:Z}, this is at least $\er_{\cD^i}(h^\star_{\cD^i})(1-\ln(1/\er_{\cD^i}(h^\star_{\cD^i})))$. We conclude $\Pr_{\cD^i}[h_1(x) \neq h_2(x)] \geq (3/4)\gamma_i/\ln(1/\gamma_i) \geq (1/2)\er_{\cD^i}(h^\star_{\cD^i})/\ln(1/\er_{\cD^i}(h^\star_{\cD^i}))$.

\item For the fifth bullet, consider a pair $h_1,h_2$ with $\Pr_{\cD^i}[h_1(\bx) \neq h_2(\bx)] > 4 \er_{\cD^i}(h^\star_{\cD^i})/\ln(1/\er_{\cD^i}(h^\star_{\cD^i}))$. From the first bullet, this implies  $\Pr_{\cD^i}[h_1(\bx) \neq h_2(\bx)] > 2\gamma_i/\ln(2/\gamma_i)$. Since $E_{i,1}$ did not occur, this implies 
\begin{eqnarray*}
\Pr_{T^i}[h_1(\bx) \neq h_2(\bx)] &\geq& 2\gamma_i/\ln(2/\gamma_i) - (1/32)\alpha(n/t,d,\delta,2\gamma_i/\ln(2/\gamma_i)) \\
&\geq& 2\gamma_i/\ln(2/\gamma_i) - (1/16)\alpha(n/t,d,\delta,\gamma_i).
\end{eqnarray*}
By definition of $\alpha$ and $Z_t$ (using $\gamma_i \geq Z_t$), this is more than $\gamma_i/\ln(1/\gamma_i)$ and hence not both of $h_1$ and $h_2$ are in $\cH^i$.
\end{itemize}
\end{proof}

\subsection{Points with few copies (proof of Lemma~\ref{lem:smallbins})}
\label{sec:reverse}
In this section, we give the proof of the following lemma
\begin{customlem}{\ref{lem:smallbins}}
Consider throwing $n$ balls independently into $u$ bins, such that the $i$'th bin is hit with probability $p_i$. Let $S$ be a subset of $m$ bins such that each bin in $S$ has $p_i = p$ for some $12/n \leq p \leq 1/2$. Then for every integer $k \leq m/C$ for a sufficiently large constant $C>0$, it holds with probability at least $1/8$ that there are at least $k$ bins in $S$ containing less than $\max\{pn - \sqrt{pn \ln(m/k)}/6,pn/2\}$ balls each.
\end{customlem}

Our proof of Lemma~\ref{lem:smallbins} follows previous work by Gr\o nlund, Kamma and Larsen~\cite{marginsGradient}. Let $S$ be a subset of $m$ bins out of $u$ bins such that each bin in $S$ is hit with probability $12/n \leq p \leq 1/2$. Now fix a bin $b\in S$ and let $\bX_1,\dots,\bX_n$ be indicator random variables, where $\bX_j$ is $1$ if the $j$'th ball is in $b$ and $0$ otherwise. We now invoke the following lemma
\begin{lemma}[Klein and Young~\cite{reversechernoff}]
\label{lem:chern}
Let $\bX_1,\dots,\bX_n$ be independent indicator random variables with success probability $p \leq 1/2$. For every $\sqrt{3/(np)} < \delta < 1/2$, 
\[
\Pr\left[\sum_i \bX_i  \leq (1-\delta)np\right] \geq \exp(-9np \delta^2).
\]
\end{lemma}
Using the requirement $m \geq Ck$ for a sufficiently large constant $C>0$, we have by Lemma~\ref{lem:chern} with $\delta = \min\{\sqrt{\ln(m/k)/(np)}/6,1/2\}$, that $\Pr[\sum_i X_i \leq np - \min\{\sqrt{np \ln(m/k)}/6, np/2\}] \geq \exp(-\ln(m/k)/4) = (k/m)^{1/4} =(m/k)^{3/4} \cdot (k/m) \geq C^{3/4}(k/m) \geq 2(k/m)$. Note that $\ln(m/k) \geq \ln(C)$ and thus $\delta > \sqrt{3/(np)}$ when $C$ is large enough. We also remark that $1/2 \geq \sqrt{3/(np)}$ since we assume $p \geq 12/n$.

Now define indicator random variables $\bY_i$ for each bin in $S$, taking the value $1$ if the number of balls in the bin is no more than $np - \min\{\sqrt{np \ln(m/k)}/6,np/2\} = \max\{np - \sqrt{np \ln(m/k)}/6 ,np/2\}$. By the argument above (and symmetry of the $\bY_i$'s), we have $\E[\bY_i] = q$ for some $q \geq 2 (k/m)$. At the same time, we have $\E[(\sum_i \bY_i)^2] = \sum_i \sum_j \E[\bY_i \bY_j]$. Since the random variables $\bY_i$ are negatively correlated, we have $\E[\bY_i \bY_j] \leq \E[\bY_i]\E[\bY_j] = q^2$ for $i \neq j$. For $i = j$, we have $\E[\bY_i^2]=\E[\bY_i] = q$. Hence $\E[\sum_i \bY_i] = mq$ and $\E[(\sum_i \bY_i)^2] \leq mq + m(m-1)q^2 \leq mq + m^2 q^2$. By Paley-Zygmund, we conclude $\Pr[\sum_i \bY_i > mq/2] \geq (1/4)(mq)^2/(mq + m^2q^2)$. Since $mq \geq 2k >1$, this is at least $1/8$. Since $mq/2 \geq k$, the conclusion follows.
 
\end{document}